\documentclass[runningheads]{llncs}
\usepackage{graphicx}
\usepackage{amsfonts,amssymb,amsmath}
\usepackage{verbatim}
\usepackage{color}
\usepackage{float}
\usepackage{graphicx}
\usepackage{latexsym}

\usepackage{epic}
\usepackage[english]{babel}
\usepackage[small]{caption}

\begin{document}
\title{Strengthening neighbourhood substitution\thanks{The author received funding from the French “Investing for the Future – PIA3”
program under the Grant agreement ANR-19-PI3A-000, in the context of the ANITI project.}}
%
%
\author{Martin C. Cooper} 
\authorrunning{M.C. Cooper}
%
\institute{ANITI, IRIT, University of Toulouse III, Toulouse, France \\ \email{cooper@irit.fr}}
\maketitle              
\begin{abstract}
Domain reduction is an essential tool for solving the constraint satisfaction problem (CSP).
In the binary CSP, neighbourhood substitution consists in eliminating a value if there exists
another value which can be substituted for it in each constraint. We show that the notion of
neighbourhood substitution can be strengthened in two distinct ways without
increasing time complexity. 
We also show the theoretical result that, 
unlike neighbourhood substitution, finding an optimal sequence of these new operations is NP-hard.

\end{abstract}
\section{Introduction}

Domain reduction is classical in constraint satisfaction. Indeed, eliminating inconsistent values
by what is now known as arc consistency~\cite{Waltz} predates the first formulation of the
constraint satisfaction problem~\cite{DBLP:journals/ai/Mackworth77}. Maintaining arc consistency, which
consists in eliminating values that can be proved inconsistent by examining a single constraint
together with the current domains of the other variables,
is ubiquitous in constraint solvers~\cite{DBLP:reference/fai/Bessiere06}.
In binary CSPs, various algorithms have been proposed for enforcing arc consistency in
$O(ed^2)$ time, where $d$ denotes maximum domain size and $e$ the number of 
constraints~\cite{DBLP:journals/ai/MohrH86,ac}.
Generic constraints on a number of variables which is unbounded are known as global constraints. 
Arc consistency can be efficiently enforced for many types of global constraints~\cite{DBLP:reference/fai/HoeveK06}. 
This has led to the development of efficient solvers providing a rich modelling language.
Stronger notions of consistency have been proposed for domain reduction which lead to more eliminations but
at greater computational cost~\cite{DBLP:reference/fai/Bessiere06,DBLP:conf/ijcai/BessiereD05,DBLP:conf/cp/WoodwardKCB12}. 

In parallel, other research has explored 
methods that preserve satisfiability of
the CSP instance but do not preserve the set of solutions. When searching for a single solution, all
but one branch of the explored search tree leads to a dead-end, and so any method for faster
detection of unsatisfiability is clearly useful. 
An important example of such methods is the addition of symmetry-breaking 
constraints~\cite{DBLP:journals/constraints/ChuS15,DBLP:reference/fai/GentPP06}.
In this paper we concentrate on domain-reduction methods.
One family of satisfiability-preserving domain-reduction
operations is value merging. For example, two values can be merged if the so-called broken triangle (BT) 
pattern does not occur on these two values~\cite{DBLP:journals/ai/CooperDMETZ16}. Other value-merging
rules have been proposed which allow less merging than BT-merging but at a lower cost~\cite{vi}
or more merging at a greater cost~\cite{wBTP,DMTCS:Naanaa}.
Another family of satisfiability-preserving domain-reduction operations are based on the elimination
of values that are not essential to obtain a solution~\cite{DBLP:conf/gcai/FreuderW17}. The basic operation
in this family which corresponds most closely to arc consistency is neighbourhood substitution:
a value $b$ can be eliminated from a domain if there is another value $a$ in the same domain such that
$b$ can be replaced by $a$ in each tuple in each constraint relation (reduced to the current domains
of the other variables)~\cite{DBLP:conf/aaai/Freuder91}. In binary CSPs,  neighbourhood substitution
can be applied until convergence in $O(ed^3)$ time~\cite{ns}. In this paper, we study notions of
substitutability which are strictly stronger than neighbourhood substitutability but which can be applied in
the same $O(ed^3)$ time complexity. We say that one elimination rule $R_1$ is stronger than (subsumes) another 
rule $R_2$ if any value in a non-trivial instance (an instance with more than one variable) 
that can be eliminated by $R_2$ can also be eliminated by $R_1$, and is strictly 
stronger (strictly subsumes) if there is also at least one non-trivial instance 
in which $R_1$ can eliminate a value that $R_2$ cannot. 
Two rules are incomparable if neither is stronger than the other.

To illustrate the strength of the new notions of substitutability that we introduce in this paper,
consider the instances shown in Figure~\ref{fig:examples}.
These instances are all globally consistent (each variable-value assignment occurs in a solution)
and neighbourhood substitution is not powerful enough to eliminate any values.
In this paper, we introduce three novel value-elimination rules, defined in Section~\ref{sec:defs}:
SS, CNS and SCSS. We will show that 
snake substitution (SS) allows us to reduce all domains to singletons in the instance in Figure~\ref{fig:examples}(a).
Using the notation $\mathcal{D}(x_i)$ for the domain of the variable $x_i$,
conditioned neighbourhood-substitution (CNS), 
allows us to eliminate value 0 from $\mathcal{D}(x_2)$ and value 2 from $\mathcal{D}(x_3)$ 
in the instance shown in Figure~\ref{fig:examples}(b), reducing the constraint between $x_2$ and $x_3$ to
a null constraint (the complete relation $\mathcal{D}(x_2) \times \mathcal{D}(x_3)$). 
Snake-conditioned snake-substitution (SCSS) subsumes both SS and CNS and allows us to reduce
all domains to singletons in the instance in Figure~\ref{fig:examples}(c) (as well as in the instances
in Figure~\ref{fig:examples}(a),(b)).

\thicklines
\setlength{\unitlength}{0.65pt}
\begin{figure}
\centering
\begin{picture}(533,120)(0,0)

\put(2,0){
\begin{picture}(160,140)(0,20)
\put(40,40){\makebox(0,0){$\bullet$}}  \put(120,40){\makebox(0,0){$\bullet$}}  
\put(40,120){\makebox(0,0){$\bullet$}}  \put(120,120){\makebox(0,0){$\bullet$}}  
\put(40,40){\line(1,0){80}} \put(40,40){\line(0,1){80}} \put(40,120){\line(1,0){80}} \put(120,40){\line(0,1){80}}
\put(40,129){\makebox(0,0){$x_1$}}  \put(120,129){\makebox(0,0){$x_2$}}
\put(120,33){\makebox(0,0){$x_3$}}  \put(40,33){\makebox(0,0){$x_4$}} 
\put(80,127){\makebox(0,0){$x_1{=}x_2$}}  \put(80,33){\makebox(0,0){$x_3{=}x_4$}} 
\put(142,78){\makebox(0,0){$x_2{\vee}x_3$}} \put(18,78){\makebox(0,0){$x_1{\vee}x_4$}}  
 \put(4,130){\makebox(0,0){(a)}}
\end{picture}}

\put(190,0){
\begin{picture}(160,140)(0,0)
\put(20,20){\makebox(0,0){$\bullet$}}  \put(140,20){\makebox(0,0){$\bullet$}}  
\put(80,120){\makebox(0,0){$\bullet$}}  
\put(20,20){\line(1,0){120}} \put(20,20){\line(3,5){60}} \put(80,120){\line(3,-5){60}}
\put(20,11){\makebox(0,0){$x_1$}}  \put(80,129){\makebox(0,0){$x_2$}}
\put(140,11){\makebox(0,0){$x_3$}}  
\put(80,12){\makebox(0,0){$\neq$}}  \put(43,80){\makebox(0,0){$\neq$}}  \put(117,80){\makebox(0,0){$\geq$}}
 \put(14,110){\makebox(0,0){(b)}}
\end{picture}}

\put(340,0){\begin{picture}(200,160)(0,0)
\put(20,140){\makebox(0,0){$\bullet$}}  \put(100,20){\makebox(0,0){$\bullet$}}  
\put(100,100){\makebox(0,0){$\bullet$}}  \put(180,140){\makebox(0,0){$\bullet$}}
\put(20,140){\line(1,0){160}} \put(20,140){\line(2,-1){80}} \put(20,140){\line(2,-3){80}}
\put(100,20){\line(2,3){80}} \put(100,100){\line(2,1){80}}  \put(100,20){\line(0,1){80}} 
\put(100,110){\makebox(0,0){$x_1$}}  \put(20,149){\makebox(0,0){$x_2$}}
\put(180,149){\makebox(0,0){$x_3$}}  \put(99,11){\makebox(0,0){$x_4$}} 
\put(101,149){\makebox(0,0){$\leq$}}  \put(142,70){\makebox(0,0){$\leq$}}  \put(58,70){\makebox(0,0){$\geq$}}
\put(108,66){\makebox(0,0){$\neq$}}  \put(60,110){\makebox(0,0){$\neq$}}  \put(140,110){\makebox(0,0){$\neq$}}
 \put(14,110){\makebox(0,0){(c)}}
\end{picture}}

\end{picture}
\caption{(a) A 4-variable CSP instance over boolean domains; (b) a 3-variable CSP instance
over domains $\{0,1,2\}$ with constraints $x_1 \neq x_2$, $x_1 \neq x_3$ and $x_2 \geq x_3$;
(c) A 4-variable CSP instance over domain $\{0,1,2,3\}$ with constraints 
$x_1 \neq x_2$, $x_1 \neq x_3$, $x_1 \neq x_4$, $x_2 \leq x_3$, $x_2 \geq x_4$ and $x_4 \leq x_3$.} \label{fig:examples}
\end{figure}
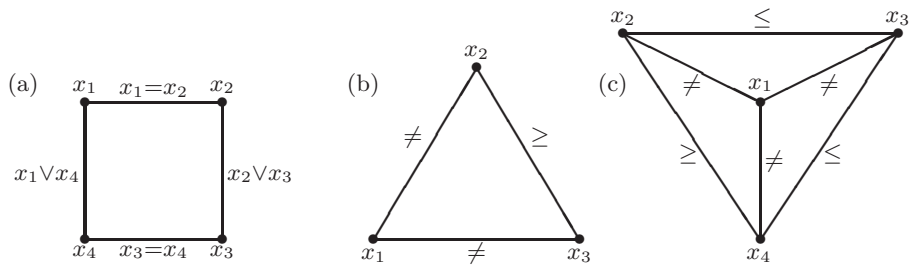
\setlength{\unitlength}{1pt}

In Section~\ref{sec:defs} we define the substitution operations SS, CNS and SCSS.
In Section~\ref{sec:proofs} we prove the validity of these three substitution operations, in the sense that 
they define satisfiability-preserving value-elimination rules. 
In Section~\ref{sec:examples} we explain in detail the examples in Figure~\ref{fig:examples}
and we give other examples from the semantic labelling of
line drawings. Section~\ref{sec:complexity} discusses the complexity 
of applying these value-elimination rules until convergence:
the time complexity 
of SS and CNS is no greater than neighbourhood substitution (NS) even though these rules are stictly stronger. However, 
unlike NS, finding an optimal sequence of value eliminations by SS or CNS is NP-hard: this is shown in Section~\ref{sec:nphard}.

\section{Definitions}  \label{sec:defs}

We study binary constraint satisfaction problems.

A \emph{binary CSP instance} $I=(X,\mathcal{D},R)$ comprises 
\begin{itemize}
\item  [$\bullet$] a set $X$ of $n$ variables $x_1,\ldots,x_n$,
\item  [$\bullet$] a domain $\mathcal{D}(x_i)$ for each variable $x_i$ ($i=1,\ldots,n$), and
\item  [$\bullet$] a binary constraint relation $R_{ij}$ for each pair of distinct variables $x_i,x_j$
($i,j \in \{1,\ldots,n\}$)
\end{itemize}
For notational convenience, we assume that there is exactly one binary relation $R_{ij}$ for each pair of variables.
Thus, if $x_i$ and $x_j$ do not constrain each other, then we consider that there is a \emph{trivial constraint}
between them with $R_{ij} = \mathcal{D}(x_i) \times \mathcal{D}(x_j)$.
Furthermore, $R_{ji}$ (viewed as a boolean matrix) is always the transpose of $R_{ij}$.
A \emph{solution} to $I$ is an $n$-tuple $s = \langle s_1,\ldots,s_n \rangle$ such that $\forall i \in \{1,\ldots,n\}$,
$s_i \in \mathcal{D}(x_i)$ and for each distinct $i,j \in \{1,\ldots,n\}$, $(s_i,s_j) \in R_{ij}$.

We say that $v_i \in D(x_i)$ has a \emph{support} at variable $x_j$ if $\exists v_j \in D(x_j)$
such that $(v_i,v_j) \in R_{ij}$. 
A binary CSP instance $I$ is \emph{arc consistent (AC)} if for all pairs of distinct variables $x_i,x_j$, 
each $v_i \in D(x_i)$ has a support at $x_j$~\cite{lecoutre}.

In the following we assume that we have a binary CSP instance $I=(X,\mathcal{D},R)$ over $n$ variables 
and, for clarity of presentation, we write $j \neq i$ as a shorthand for $j \in \{1,\ldots,n\} \setminus \{i\}$.
We use the notation $b \xrightarrow{\scriptstyle{ij}} a$ for 
\[\forall c \in \mathcal{D}(x_j), \ (b,c) \in R_{ij} \Rightarrow (a,c) \in R_{ij}
\]
(i.e. $a$ can be substituted for $b$ in any tuple $(b,c) \in R_{ij}$).

\begin{definition}\cite{DBLP:conf/aaai/Freuder91}  \ \label{def:NS}
Given two values $a,b \in \mathcal{D}(x_i)$, $b$ is \emph{neighbourhood substitutable (NS)} by $a$
if $\forall j \neq i$, $b \xrightarrow{\scriptstyle{ij}} a$.
\end{definition}

\thicklines
\setlength{\unitlength}{1.5pt}
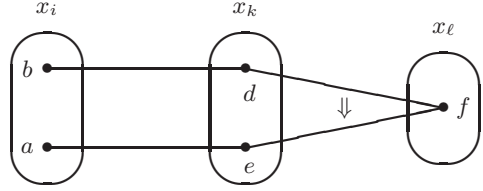
\begin{figure}
\centering
\begin{picture}(120,45)(0,0)

\put(10,10){\makebox(0,0){$\bullet$}}  \put(10,30){\makebox(0,0){$\bullet$}}  
\put(60,10){\makebox(0,0){$\bullet$}}  \put(60,30){\makebox(0,0){$\bullet$}}  
\put(110,20){\makebox(0,0){$\bullet$}}  \put(10,20){\oval(18,38)}  \put(60,20){\oval(18,38)}   \put(110,20){\oval(18,28)} 
\put(60,10){\line(5,1){50}}  \put(60,30){\line(5,-1){50}} 
\put(10,45){\makebox(0,0){$x_i$}}  \put(60,45){\makebox(0,0){$x_k$}}
\put(110,40){\makebox(0,0){$x_{\ell}$}}  \put(61,24){\makebox(0,0){$d$}} 
\put(61,5){\makebox(0,0){$e$}} \put(115,20){\makebox(0,0){$f$}}  \put(85,20){\makebox(0,0){$\Downarrow$}}
\put(5,10){\makebox(0,0){$a$}}  \put(5,30){\makebox(0,0){$b$}}  
\put(10,10){\line(5,0){50}} \put(10,30){\line(5,0){50}}  

\end{picture}
\caption{An illustration of the definition of $b \overset{ik}{\rightsquigarrow} a$.} \label{fig:snake}
\end{figure}
\setlength{\unitlength}{1pt}

It is well known and indeed fairly obvious that eliminating a neighbourhood substitutable value does not
change the satisfiability of a binary CSP instance. We will now define stronger notions of
substitutability. The proofs that these are indeed valid value-elimination rules 
are not directly obvious and hence are delayed until Section~\ref{sec:proofs}. 
We use the notation $b \overset{ik}{\rightsquigarrow} a$ for 
\[\forall d \in \mathcal{D}(x_k), \ (b,d) \in R_{ik} \ \Rightarrow \
\exists e \in \mathcal{D}(x_k) ( (a,e) \in R_{ik} \ \wedge \ \forall \ell \notin \{i,k\}, d \xrightarrow{\scriptstyle{k \ell}} e ).
\]
This is illustrated in Figure~\ref{fig:snake}, in which ovals represent domains, bullets represent values, a
line joining two values means that these two values are compatible (so, for example, $(a,e) \in R_{ik}$),
and the $\Downarrow$ means that $(d,f) \in R_{k \ell} \Rightarrow (e,f) \in R_{k \ell}$. 
Since $e$ in this definition is a function of $i,k,a$ and $d$, if necessary, we will write $e(i,k,a,d)$ instead of $e$.
In other words, the notation $b \overset{ik}{\rightsquigarrow} a$ means that $a$ can be substituted for $b$  in any tuple 
$(b,d) \in R_{ik}$  provided we also replace $d$ by $e(i,k,a,d)$. It is clear that $b \xrightarrow{\scriptstyle{ik}} a$
implies  $b \overset{ik}{\rightsquigarrow} a$ since it suffices to set $e(i,k,a,d) = d$ since, trivially,
$d \xrightarrow{\scriptstyle{k \ell}} d$ for all $\ell \notin \{i,k\}$.
In Figure~\ref{fig:examples}(a), the value $0 \in \mathcal{D}(x_1)$ is snake substitutable by $1$: we have  
$0 \overset{12}{\rightsquigarrow} 1$ by taking 
$e(1,2,1,0)=1$ (where the arguments of $e(i,k,a,d)$ are as shown in Figure~\ref{fig:snake}), since
$(1,1) \in R_{12}$ and $0 \xrightarrow{\scriptstyle{23}} 1$; and $0 \overset{14}{\rightsquigarrow} 1$ since
$0 \xrightarrow{\scriptstyle{14}} 1$.
Indeed, by a similar argument,
the value $0$ is snake substitutable by $1$ in each domain.

\begin{definition}  \label{def:SS}
Given two values $a,b \in \mathcal{D}(x_i)$, $b$ is \emph{snake substitutable (SS}) by $a$
if $\forall k \neq i$, $b \overset{ik}{\rightsquigarrow} a$.
\end{definition}

\thicklines
\setlength{\unitlength}{1.5pt}
\begin{figure}
\centering
\begin{picture}(120,40)(0,0)

\put(10,20){\makebox(0,0){$\bullet$}}  \put(60,10){\makebox(0,0){$\bullet$}}  \put(60,30){\makebox(0,0){$\bullet$}}  
\put(110,20){\makebox(0,0){$\bullet$}}  \put(10,20){\oval(18,28)}  \put(60,20){\oval(18,38)}   \put(110,20){\oval(18,28)} 
\put(10,20){\line(5,1){50}} \put(10,20){\line(5,-1){50}} \put(60,10){\line(5,1){50}}  \put(60,30){\line(5,-1){50}} 
\put(10,40){\makebox(0,0){$x_j$}}  \put(60,45){\makebox(0,0){$x_i$}}
\put(110,40){\makebox(0,0){$x_k$}}  \put(5,20){\makebox(0,0){$c$}} 
\put(59,5){\makebox(0,0){$b$}}  \put(59,25){\makebox(0,0){$a$}} \put(115,20){\makebox(0,0){$d$}} 
 \put(85,20){\makebox(0,0){$\Uparrow$}} 

\end{picture}
\caption{An illustration of the definition of conditioned neighbourhood-substitutability of $b$ by $a$
(conditioned by $x_j$).} \label{fig:CNS}
\end{figure}
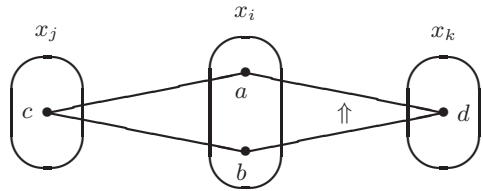
\setlength{\unitlength}{1pt}

In the following two definitions, $b$ can be eliminated from $\mathcal{D}(x_i)$ because it can be substituted by some other
value in $\mathcal{D}(x_i)$, but this value is a function of the value assigned to another variable $x_j$.
Definition~\ref{def:CNS} is illustrated in Figure~\ref{fig:CNS}.

\begin{definition}  \label{def:CNS}
Given  $b \in \mathcal{D}(x_i)$, $b$ is \emph{conditioned neighbourhood-substitutable (CNS)} if
for some $j \neq i$, $\forall c \in \mathcal{D}(x_j)$ with $(b,c) \in R_{ij}$, $\exists a \in \mathcal{D}(x_i) \setminus \{b\}$
such that $((a,c) \in R_{ij} \ \wedge \ \forall k \notin \{i,j\},  b \xrightarrow{\scriptstyle{ik}} a)$.
\end{definition}

A CNS value $b \in \mathcal{D}(x_i)$ is substitutable by a value $a \in \mathcal{D}(x_i)$ where $a$ is a
function of the value $c$ assigned to some other variable $x_j$. In Figure~\ref{fig:examples}(b), the value 
$0 \in \mathcal{D}(x_2)$ is conditioned neighbourhood-substitutable (CNS) with $x_1$ as the conditioning variable
(i.e. $j=1$ in Definition~\ref{def:CNS}): for the
assignments of $0$ or $1$ to $x_1$, we can take $a=2$ since $0 \xrightarrow{\scriptstyle{23}} 2$,
and for the assignment $2$ to $x_1$, we can take $a=1$ since $0 \xrightarrow{\scriptstyle{23}} 1$.
By a symmetrical argument, the value $2 \in \mathcal{D}(x_3)$ is CNS, again with $x_1$ as the conditioning variable.
We can note that in the resulting CSP instance, after eliminating $0$ from $\mathcal{D}(x_2)$
and $2$ from $\mathcal{D}(x_3)$, all domains can be reduced to singletons by applying snake substitutability.

Observe that CNS subsumes arc consistency;
if a value $b \in \mathcal{D}(x_i)$ has no support $c$ in $\mathcal{D}(x_j)$, then $b$ is trivially CNS
(conditioned by the variable $x_j$). It is easy to see from their definitions that SS and CNS both subsume NS
(in instances with more than one variable), but that neither NS nor SS subsume arc consistency.

We now integrate the notion of snake substitutability in two ways in the definition of CNS:
the value $d$ (see Figure~\ref{fig:CNS}) assigned to a variable $k \notin \{i,j\}$ 
may be replaced by a value $e$ (as in the definition of $b \overset{ik}{\rightsquigarrow} a$, above), 
but the value $c$ (see Figure~\ref{fig:CNS}) 
assigned to the conditioning variable $x_j$ may also be replaced by a value $g$. This is illustrated in Figure~\ref{fig:SCSS}.

\thicklines
\setlength{\unitlength}{1.5pt}
\begin{figure}
\centering
\begin{picture}(220,50)(0,0)
 
\put(10,20){\makebox(0,0){$\bullet$}}    
\put(60,10){\makebox(0,0){$\bullet$}}  \put(60,30){\makebox(0,0){$\bullet$}}  
\put(110,10){\makebox(0,0){$\bullet$}} \put(110,30){\makebox(0,0){$\bullet$}} 
\put(160,10){\makebox(0,0){$\bullet$}} \put(160,30){\makebox(0,0){$\bullet$}} 
\put(210,20){\makebox(0,0){$\bullet$}} 
\put(10,20){\oval(18,28)}  \put(60,20){\oval(18,38)}  \put(110,20){\oval(18,38)} \put(160,20){\oval(18,38)}  
\put(210,20){\oval(18,28)} 
\put(10,20){\line(5,1){50}} \put(10,20){\line(5,-1){50}} 
\put(160,10){\line(5,1){50}}  \put(160,30){\line(5,-1){50}} 
\put(10,40){\makebox(0,0){$x_m$}}  \put(60,45){\makebox(0,0){$x_j$}}
\put(110,45){\makebox(0,0){$x_i$}}  \put(160,45){\makebox(0,0){$x_k$}} \put(210,45){\makebox(0,0){$x_{\ell}$}} 
\put(5,20){\makebox(0,0){$h$}}  \put(61,25){\makebox(0,0){$c$}} 
\put(62,5){\makebox(0,0){$g$}}   
\put(161,25){\makebox(0,0){$d$}} \put(161,5){\makebox(0,0){$e$}}  \put(215,20){\makebox(0,0){$f$}} 
 \put(180,20){\makebox(0,0){$\Downarrow$}}  \put(40,20){\makebox(0,0){$\Downarrow$}} 
\put(60,10){\line(5,0){50}}  \put(60,30){\line(5,0){50}} 
\put(110,10){\line(5,0){50}}  \put(110,30){\line(5,0){50}} 
\put(109,5){\makebox(0,0){$a$}} \put(108,25){\makebox(0,0){$b$}}

\end{picture}
\caption{An illustration of 
snake-conditioned snake-substitutability of $b$ by $a$.} \label{fig:SCSS}
\end{figure}
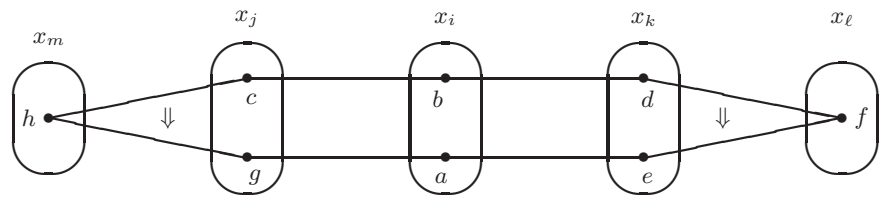
\setlength{\unitlength}{1pt}

\begin{definition}  \label{def:SCSS}
A value $b \in \mathcal{D}(x_i)$ is \emph{snake-conditioned snake-substitutable \linebreak (SCSS)} if
for some $j \neq i$, $\forall c \in \mathcal{D}(x_j)$ with $(b,c) \in R_{ij}$, $\exists a \in \mathcal{D}(x_i) \setminus \{b\}$
such that $( \forall k \notin \{i,j\}, b \overset{ik}{\rightsquigarrow} a \ \wedge \ 
(\exists g \in \mathcal{D}(x_j) ( (a,g) \in R_{ij} \ \wedge \ \forall m \notin \{i,j\}, c \xrightarrow{\scriptstyle{jm}} g)) )$.
\end{definition}

In Figure~\ref{fig:examples}(c), the value $3 \in \mathcal{D}(x_1)$ is snake-conditioned snake-substitutable (SCSS)
with $x_2$ as the conditioning variable:
for the assignment of $0$ or $2$ to $x_2$, we can take $a=1$ since $3 \overset{13}{\rightsquigarrow} 1$
(taking $e(1,3,1,d)=3$ for $d=0,1,2$) and $3 \overset{14}{\rightsquigarrow} 1$
(taking $e(1,4,1,d)=0$ for $d=0,1,2$), and for the assignment of $1$ to $x_2$, 
we can take $a=2$ since $3 \overset{13}{\rightsquigarrow} 2$
(again taking $e(1,3,2,d)=3$ for $d=0,1,2$) and $3 \overset{14}{\rightsquigarrow} 2$
(again taking $e(1,4,2,d)=0$ for $d=0,1,2$). 
By similar arguments, all domains can be reduced to singletons following the SCSS elimination 
of values in the following order: $0$ from $\mathcal{D}(x_1)$, $0$, $1$ and $2$ from $\mathcal{D}(x_3)$, 
$0$, $1$ and $2$ from $\mathcal{D}(x_2)$, $1$, $2$ and $3$ from $\mathcal{D}(x_4)$ and $2$ from $\mathcal{D}(x_1)$.

We can see that SCSS subsumes CNS by setting $g=c$ in Definition~\ref{def:SCSS} and by recalling that 
$b \xrightarrow{\scriptstyle{ik}} a$ implies that $b \overset{ik}{\rightsquigarrow} a$. It is a bit more subtle to
see that SCSS subsumes SS: if $b$ is snake substitutable by some value $a$, it suffices to choose $a$ in Definition~\ref{def:SCSS} 
to be this value (which is thus constant, i.e. not dependent on the value of $c$), then  the snake substitutability of $b$ by $a$
implies that $b \overset{ik}{\rightsquigarrow} a$ for all $k \neq i,j$ 
and $b \overset{ij}{\rightsquigarrow} a$, which in turn implies that 
$(a,g) \in R_{ij} \ \wedge \ \forall m \notin \{i,j\}, c \xrightarrow{\scriptstyle{jm}} g$ for
 $g = e(i,j,a,c)$; thus $b$ is snake-conditioned snake-substitutable.

\section{Value elimination}  \label{sec:proofs} 

It is well-known that NS is a valid value-elimination property, in the sense that if $b \in \mathcal{D}(x_i)$
is neighbourhood substitutable by $a$ then $b$ can be eliminated from $\mathcal{D}(x_i)$ without
changing the satisfiability of the CSP instance~\cite{DBLP:conf/aaai/Freuder91}.
In this section we show that SCSS is a valid value-elimination property. Since SS and CNS are subsumed
by SCSS, it follows immediately that SS and CNS are also valid value-elimination properties.

\begin{theorem}
In a binary CSP instance $I$, if $b \in \mathcal{D}(x_i)$ is snake-conditioned snake-substitutable then $b$ can be 
eliminated from $\mathcal{D}(x_i)$ without changing the satisfiability of the instance.
\end{theorem}

\begin{proof}
By Definition~\ref{def:SCSS}, for some $j \neq i$, $\forall c \in \mathcal{D}(x_j)$ with $(b,c) \in R_{ij}$, 
$\exists a \in \mathcal{D}(x_i) \setminus \{b\}$ such that 
\begin{eqnarray}
& \forall k \notin \{i,j\}, b \overset{ik}{\rightsquigarrow} a   \label{line1} \\ 
\wedge & \ \ 
\exists g \in \mathcal{D}(x_j) ( (a,g) \in R_{ij} \ \wedge \ \forall m \notin \{i,j\}, c \xrightarrow{\scriptstyle{jm}} g). \label{line2}
\end{eqnarray}
We will only apply this definition for fixed $i,j$, and for fixed values $a$ and $c$, so we can consider $g$ as a constant
(even though it is actually a function of $i,j,a,c$).
Let $s = \langle s_1,\ldots,s_n \rangle$ be a solution to $I$ with $s_i=b$. It suffices to show that there is another
solution $t = \langle t_1,\ldots,t_n \rangle$ with $t_i \neq b$. Consider $c=s_j$.
Since $s$ is a solution, we know that $(b,c) = (s_i,s_j) \in R_{ij}$. Thus, according to the above definition of SCSS, 
there is a value $a \in  \mathcal{D}(x_i)$ that can replace $b$ (conditioned by the assignment $x_j=c=s_j$)
in the sense that (\ref{line1}) and (\ref{line2}) are satisfied.
Now, for each $k \notin \{i,j\}$, $b \overset{ik}{\rightsquigarrow} a$, i.e. 
\[
\forall d \in \mathcal{D}(x_k), \ (b,d) \in R_{ik} \ \Rightarrow \
\exists e \in \mathcal{D}(x_k) ( (a,e) \in R_{ik} \ \wedge \ \forall \ell \notin \{i,k\}, d \xrightarrow{\scriptstyle{k \ell}} e ).
\]
Recall that $e$ is a function of $i,k,a$ and $d$. But we will only consider fixed $i,a$ and a unique
value of $d$ dependant on $k$, so we will write
$e(k)$ for brevity. Indeed, setting $d=s_k$ we can deduce from $(b,d)=(s_i,s_k) \in R_{ik}$
(since $s$ is a solution) that
\begin{equation}
\forall k \neq i,j, \ \ 
\exists e(k) \in \mathcal{D}(x_k) ( (a,e(k)) \in R_{ik} \ \wedge \ \forall \ell \notin \{i,k\}, s_k \xrightarrow{\scriptstyle{k \ell}} e(k) ).
\label{line3}
\end{equation}

Define the $n$-tuple $t$ as follows:
\begin{equation*}
t_r=
\begin{cases}
a & \text{if $r = i$}\\
s_r & \text{if $r \neq i \ \wedge \ (a,s_r) \in R_{ir}$}\\
g & \text{if $r = j \ \wedge \ (a,s_r) \notin R_{ir}$}\\
e(r) & \text{if $r \neq i,j \ \wedge \ (a,s_r) \notin R_{ir}$}
\end{cases}
\end{equation*}
Clearly $t_i \neq b$ and $t_r \in \mathcal{D}(x_r)$ for all $r \in \{1,\ldots,n\}$.
To prove that $t$ is a solution, it remains to show that all binary constraints are satisfied,
i.e. that $(t_k,t_r) \in R_{kr}$ for all distinct $k,r \in \{1,\ldots,n\}$. There are three cases: 
(1) $k=i$, $r \neq i$, (2) $k=j$, $r \neq i,j$, (3) $k,r \neq i,j$.
\begin{itemize}
\item[(1)] There are three subcases: (a) $r=j$ and $(a,s_j) \notin R_{ij}$, (b) $r \neq i$ and $(a,s_r) \in R_{ir}$,
(c) $r \neq i,j$ and $(a,s_r) \notin R_{ir}$. In case (a), $t_i=a$ and $t_j=g$, so from equation~\ref{line2},
we have $(t_i,t_r) = (a,g) \in R_{ij}$. In case (b), $t_i=a$ and $t_r = s_r$ and so, trivially, $(t_i,t_r) = (a,s_r) \in R_{ir}$.
In case (c), $t_i = a$ and $t_r = e(r)$, so from equation~\ref{line3}, we have $(t_i,t_r) = (a,e(r)) \in R_{ir}$.

\item[(2)] There are four subcases: (a) $(a,s_r) \in R_{ir}$ and $(a,s_j) \in R_{ij}$,
(b) $(a,s_r) \notin R_{ir}$ and $(a,s_j) \in R_{ij}$,
(c) $(a,s_r) \in R_{ir}$ and $(a,s_j) \notin R_{ij}$,
(d) $(a,s_r) \notin R_{ir}$ and $(a,s_j) \notin R_{ij}$.
In case (a), $t_j=s_j$ and $t_r=s_r$, so $(t_j,t_r) \in R_{jr}$ since $s$ is a solution.
In case (b), $t_j=s_j$ and $t_r = e(r)$; setting $k=r$, $\ell=j$ in equation~\ref{line3}, we have 
$(t_j,t_r) = (s_j,e(r)) \in R_{jr}$ since $(s_j,s_r) \in R_{jr}$. In case (c), $t_j=g$ and $t_r=s_r$;
setting $c=s_j$ and $m=r$ in equation~\ref{line2} we can deduce that $(t_j,t_r) = (g,s_r) \in R_{jr}$
since $(s_j,s_r) \in R_{jr}$. In case (d), $t_j=g$ and $t_r=e(r)$. By the same argument as in case 2(b),
we know that $(s_j,e(r)) \in R_{jr}$, and then setting $c=s_j$ and $m=r$ in equation~\ref{line2},
we can deduce that $(t_j,t_r) = (g,e(r)) \in R_{jr}$.

\item[(3)] There are three essentially distinct subcases: (a) $(a,s_r) \in R_{ir}$ and $(a,s_k) \in R_{ik}$,
(b) $(a,s_r) \notin R_{ir}$ and $(a,s_k) \in R_{ik}$,
(c) $(a,s_r) \notin R_{ir}$ and $(a,s_k) \notin R_{ik}$.
In cases (a) and (b) we can deduce $(t_k,t_r) \in R_{kr}$ by the same arguments as in cases 2(a) and 2(b), above.
In case (c), $t_k=e(k)$ and $t_r=e(k)$. Setting $\ell=r$ in equation~\ref{line3},
we have $s_k \xrightarrow{\scriptstyle{kr}} e(k)$ from which we can deduce that
$(e(k),s_r) \in R_{kr}$ since $(s_k,s_r) \in R_{kr}$. Reversing the roles of $k$ and $r$ 
in equation~\ref{line3} (which is possible since they are distinct and both different to $i$ and $j$), 
we also have that $s_r \xrightarrow{\scriptstyle{rk}} e(r)$. We can then deduce that 
$(t_k,t_r) = (e(k),e(r)) \in R_{kr}$ since we have just shown that $(e(k),s_r) \in R_{kr}$.
\end{itemize}
We have thus shown that any solution $s$ with $s_i=b$ can be transformed into another solution $t$ that does not assign
the value $b$ to $x_i$ and hence that the elimination of $b$ from $\mathcal{D}(x_i)$ preserves satisfiability.
\end{proof}

\begin{corollary}  \label{cor:SS+CNS}
In a binary CSP instance $I$, if $b \in \mathcal{D}(x_i)$ is snake-substitutable or conditioned neighbourhood
substitutable, then $b$ can be eliminated from $\mathcal{D}(x_i)$ without changing the satisfiability of the instance.
\end{corollary}

\section{Examples}  \label{sec:examples}

We illustrate the potential of SS, CNS and SCSS using the examples given in Figure~\ref{fig:examples}.
In Figure~\ref{fig:examples}(a), the value $0 \in \mathcal{D}(x_1)$ is snake substitutable by $1$: we have  
$0 \overset{12}{\rightsquigarrow} 1$ by taking 
$e(1,2,1,0)=1$ (where the arguments of $e(i,k,a,d)$ are as shown in Figure~\ref{fig:snake}), since
$(1,1) \in R_{12}$ and $0 \xrightarrow{\scriptstyle{23}} 1$; and $0 \overset{14}{\rightsquigarrow} 1$ since
$0 \xrightarrow{\scriptstyle{14}} 1$.
Indeed, by a similar argument,
the value $0$ is snake substitutable by $1$ in each domain. In Figure~\ref{fig:examples}(b), the value 
$0 \in \mathcal{D}(x_2)$ is conditioned neighbourhood-substitutable (CNS) with $x_1$ as the conditioning variable
(i.e. $j=1$ in Definition~\ref{def:CNS}): for the
assignments of $0$ or $1$ to $x_1$, we can take $a=2$ since $0 \xrightarrow{\scriptstyle{23}} 2$,
and for the assignment $2$ to $x_1$, we can take $a=1$ since $0 \xrightarrow{\scriptstyle{23}} 1$.
By a symmetrical argument, the value $2 \in \mathcal{D}(x_3)$ is CNS, again with $x_1$ as the conditioning variable.
We can note that in the resulting CSP instance, after eliminating $0$ from $\mathcal{D}(x_2)$
and $2$ from $\mathcal{D}(x_3)$, all domains can be reduced to singletons by applying snake substitutability.

In Figure~\ref{fig:examples}(c), the value $3 \in \mathcal{D}(x_1)$ is snake-conditioned snake-substitutable (SCSS)
with $x_2$ as the conditioning variable:
for the assignment of $0$ or $2$ to $x_2$, we can take $a=1$ since $3 \overset{13}{\rightsquigarrow} 1$
(taking $e(1,3,1,d)=3$ for $d=0,1,2$) and $3 \overset{14}{\rightsquigarrow} 1$
(taking $e(1,4,1,d)=0$ for $d=0,1,2$), and for the assignment of $1$ to $x_2$, 
we can take $a=2$ since $3 \overset{13}{\rightsquigarrow} 2$
(again taking $e(1,3,2,d)=3$ for $d=0,1,2$) and $3 \overset{14}{\rightsquigarrow} 2$
(again taking $e(1,4,2,d)=0$ for $d=0,1,2$). 
By similar arguments, all domains can be reduced to singletons following the SCSS elimination 
of values in the following order: $0$ from $\mathcal{D}(x_1)$, $0$, $1$ and $2$ from $\mathcal{D}(x_3)$, 
$0$, $1$ and $2$ from $\mathcal{D}(x_2)$, $1$, $2$ and $3$ from $\mathcal{D}(x_4)$ and $2$ from $\mathcal{D}(x_1)$.

\thicklines \setlength{\unitlength}{1pt}
\newsavebox{\ru}
\savebox{\ru}(20,10){
\begin{picture}(20,10)(0,0)
\put(16,8){\line(-1,0){13}}  \put(16,8){\line(-5,-6){9}}
\end{picture}
}
\newsavebox{\ld}
\savebox{\ld}(20,10){
\begin{picture}(20,10)(0,0)
\put(4,2){\line(1,0){13}}  \put(4,2){\line(5,6){9}}
\end{picture}
}

\thicklines \setlength{\unitlength}{1.2pt}
\begin{figure}
\centering
\begin{picture}(280,180)(0,0)
\put(20,100){
\begin{picture}(40,40)(0,0)
\multiput(0,10)(20,-10){2}{\line(0,1){20}}
\multiput(0,10)(0,20){2}{\line(2,-1){20}}
\multiput(20,0)(0,20){2}{\line(2,1){20}} \put(0,30){\line(2,1){20}}
\put(40,10){\line(0,1){20}} \put(20,40){\line(2,-1){20}}
\put(13,16){\makebox(0,0){$A$}}
\end{picture}
}
\put(100,100){
\begin{picture}(80,50)(0,0)
\multiput(0,10)(20,-10){2}{\line(0,1){20}}
\multiput(0,10)(0,20){2}{\line(2,-1){20}}
\multiput(20,0)(0,20){2}{\line(2,1){20}} \put(0,30){\line(2,1){40}}
\put(40,0){\begin{picture}(40,40)(0,0)
\multiput(0,10)(20,-10){2}{\line(0,1){20}}
\multiput(0,10)(0,20){2}{\line(2,-1){20}}
\multiput(20,0)(0,20){2}{\line(2,1){20}}
\end{picture}}
\put(80,10){\line(0,1){20}} \put(40,50){\line(2,-1){40}}
\put(40,37){\makebox(0,0){$B$}}
\put(15,12){\makebox(0,0){$+$}} \put(35,17){\makebox(0,0){$-$}} \put(10,35){\usebox{\ru}} \put(24,2){\usebox{\ld}}
\end{picture}
}
\put(200,100){
\begin{picture}(65,80)(0,0)
\multiput(0,20)(0,20){2}{\line(2,-1){45}}
\multiput(0,20)(45,-22.5){2}{\line(0,1){20}}
\multiput(45,-2.5)(0,20){2}{\line(2,1){20}}
\put(65,7.5){\line(0,1){20}} \put(20,50){\line(2,-1){45}}
\put(0,40){\line(2,1){20}} \put(60,30){\line(0,1){40}}
\multiput(20,50)(20,-10){2}{\line(0,1){20}}
\multiput(20,70)(20,-10){2}{\line(2,1){20}}
\multiput(20,70)(20,10){2}{\line(2,-1){20}}
\put(36,34){\makebox(0,0){$C$}}
\end{picture}
}
\put(0,0){
\begin{picture}(80,70)(0,0)
\multiput(0,10)(20,-10){2}{\line(0,1){20}}
\multiput(0,10)(0,20){2}{\line(2,-1){20}}
\multiput(20,0)(0,20){2}{\line(2,1){20}} \put(0,30){\line(2,1){20}}
\put(20,30){\begin{picture}(40,40)(0,0)
\multiput(0,10)(20,-10){2}{\line(0,1){20}}
\multiput(0,10)(0,20){2}{\line(2,-1){20}}
\multiput(20,0)(0,20){2}{\line(2,1){20}} \put(0,30){\line(2,1){20}}
\put(40,10){\line(0,1){20}} \put(20,40){\line(2,-1){20}}
\end{picture}}
\put(40,0){\begin{picture}(40,40)(0,0)
\multiput(0,10)(20,-10){2}{\line(0,1){20}}
\multiput(0,10)(0,20){2}{\line(2,-1){20}}
\multiput(20,0)(0,20){2}{\line(2,1){20}} \put(40,10){\line(0,1){20}}
\put(20,40){\line(2,-1){20}}
\end{picture}}
\put(27.5,30){\makebox(0,0){$D$}}
\end{picture}
}
\put(100,0){
\begin{picture}(80,80)(0,0)
\multiput(0,20)(0,20){2}{\line(2,-1){40}}
\multiput(0,20)(40,-20){2}{\line(0,1){20}}
\multiput(40,0)(0,20){2}{\line(2,1){40}} \put(0,40){\line(2,1){20}}
\put(60,50){\line(2,-1){20}} \put(80,20){\line(0,1){20}}
\put(20,40){\begin{picture}(40,40)(0,0)
\multiput(0,10)(20,-10){2}{\line(0,1){20}}
\multiput(0,10)(0,20){2}{\line(2,-1){20}}
\multiput(20,0)(0,20){2}{\line(2,1){20}} \put(0,30){\line(2,1){20}}
\put(40,10){\line(0,1){20}} \put(20,40){\line(2,-1){20}}
\end{picture}}
\put(40,33){\makebox(0,0){$E$}}
\end{picture}
}
\put(200,0){
\begin{picture}(80,80)(0,0)
\multiput(0,20)(40,60){2}{\line(2,-1){40}}
\multiput(0,60)(40,-60){2}{\line(2,1){40}}
\multiput(0,20)(80,0){2}{\line(0,1){40}}
\multiput(0,60)(40,0){2}{\line(2,-1){20}}
\multiput(20,50)(40,0){2}{\line(2,1){20}}
\multiput(20,30)(20,10){2}{\line(0,1){20}}
\multiput(20,30)(20,10){2}{\line(2,-1){20}}
\multiput(20,30)(20,-10){2}{\line(2,1){20}}
\multiput(40,0)(20,30){2}{\line(0,1){20}}
\put(40,31.5){\makebox(0,0){$F$}}
\end{picture}
}
\end{picture} 
\caption{The six different types of trihedral vertices: $A$, $B$, $C$, $D$, $E$, $F$.} \label{fig:vertices}
\end{figure}
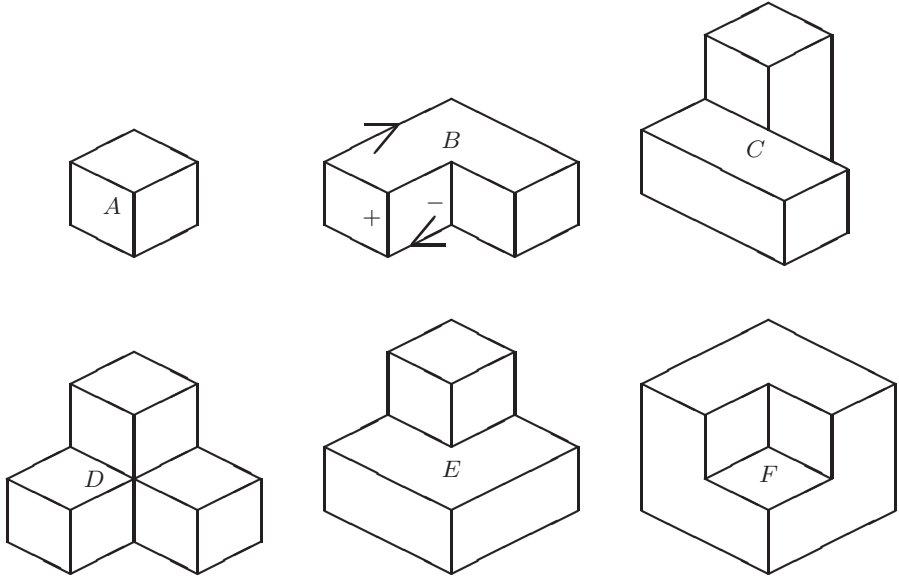

\thicklines
\setlength{\unitlength}{1pt}
\begin{figure}
\centering
\begin{picture}(340,370)(0,-35)
\thicklines
\put(5,250){ 
\begin{picture}(80,60)(0,0)
\put(40,20){\line(1,1){30}} \put(40,20){\line(-1,1){30}}
\put(60,40){\line(0,-1){10}} \put(60,40){\line(-1,0){10}}
\put(30,40){\makebox(0,0){$+$}}
\end{picture}
}
\put(85,250){ 
\begin{picture}(80,60)(0,0)
\put(40,20){\line(1,1){30}} \put(40,20){\line(-1,1){30}}
\put(30,30){\line(0,1){10}} \put(30,30){\line(-1,0){10}}
\put(50,40){\makebox(0,0){$+$}}
\end{picture}
}
\put(165,250){ 
\begin{picture}(80,60)(0,0)
\put(40,20){\line(1,1){30}} \put(40,20){\line(-1,1){30}}
\put(30,30){\line(0,1){10}} \put(30,30){\line(-1,0){10}}
\put(60,40){\line(0,-1){10}} \put(60,40){\line(-1,0){10}}
\end{picture}
}
\put(5,195){ 
\begin{picture}(240,60)(0,0)
\multiput(40,20)(80,0){3}{\line(1,1){30}}
\multiput(40,20)(80,0){3}{\line(-1,1){30}}
\multiput(50,30)(80,0){2}{\line(0,1){10}}
\multiput(50,30)(80,0){2}{\line(1,0){10}}
\multiput(20,40)(160,0){2}{\line(0,-1){10}}
\multiput(20,40)(160,0){2}{\line(1,0){10}}
\put(210,40){\makebox(0,0){$-$}} \put(110,40){\makebox(0,0){$-$}}
\end{picture}
}
\put(160,120){ 
\begin{picture}(80,80)(0,0)
\put(40,20){\line(0,1){30}} \put(40,50){\line(3,2){30}}
\put(40,50){\line(-3,2){30}} \put(50,35){\makebox(0,0){$+$}}
\put(60,55){\makebox(0,0){$+$}} \put(20,55){\makebox(0,0){$+$}}
\end{picture}
}
\put(0,120){ 
\begin{picture}(320,80)(0,0)
\multiput(40,20)(80,0){2}{\line(0,1){30}}
\multiput(40,50)(80,0){2}{\line(3,2){30}}
\multiput(40,50)(80,0){2}{\line(-3,2){30}}
\put(50,35){\makebox(0,0){$-$}}
\multiput(60,55)(80,0){2}{\makebox(0,0){$-$}}
\put(20,55){\makebox(0,0){$-$}} \put(120,40){\line(1,-1){10}}
\put(120,40){\line(-1,-1){10}} \put(105,60){\line(1,0){10}}
\put(105,60){\line(0,-1){10}}
\end{picture}
}
\put(0,55){ 
\begin{picture}(80,80)(0,0)
\multiput(40,20)(80,0){2}{\line(0,1){40}}
\multiput(40,20)(80,0){2}{\line(3,2){30}}
\multiput(40,20)(80,0){2}{\line(-3,2){30}}
\multiput(50,50)(80,0){2}{\makebox(0,0){$+$}}
\put(60,25){\makebox(0,0){$-$}} \put(20,25){\makebox(0,0){$-$}}
\put(105,30){\line(1,0){10}} \put(105,30){\line(0,-1){10}}
\put(135,30){\line(1,0){10}} \put(135,30){\line(0,1){10}}
\end{picture}
}
\put(160,55){ 
\begin{picture}(80,80)(0,0)
\put(40,20){\line(0,1){40}} \put(40,20){\line(3,2){30}}
\put(40,20){\line(-3,2){30}} \put(50,50){\makebox(0,0){$-$}}
\put(60,25){\makebox(0,0){$+$}} \put(20,25){\makebox(0,0){$+$}}
\end{picture}
}
\put(45,-25){ 
\begin{picture}(60,75)(-10,-17.5)
\put(-10,-2.5){\line(4,3){60}} \put(-10,42.5){\line(4,-3){60}}
\put(20,-17.5){\line(0,1){75}} \put(15,40){\makebox(0,0){$+$}}
\put(15,0){\makebox(0,0){$-$}} \put(-2,10){\makebox(0,0){$+$}}
\put(-2,30){\makebox(0,0){$-$}} \put(42,10){\makebox(0,0){$+$}}
\put(42,30){\makebox(0,0){$-$}}
\end{picture}
} \put(132.5,-25){
\begin{picture}(60,75)(-10,-17.5)
\put(20,20){\line(4,3){30}} \put(-10,42.5){\line(4,-3){30}}
\put(20,-17.5){\line(0,1){75}} \put(15,40){\makebox(0,0){$+$}}
\put(15,0){\makebox(0,0){$-$}} \put(40,35){\line(-1,0){10}}
\put(40,35){\line(-1,-4){2.5}} \put(4,32){\line(-1,0){10}}
\put(4,32){\line(-1,4){2.5}}
\end{picture}
}
\put(255,250){ 
\begin{picture}(80,80)(0,0)
\put(47,30){\makebox(0,0){?}} \put(40,20){\line(0,1){30}}
\put(10,50){\line(1,0){60}} \multiput(20,50)(30,0){2}{\line(5,3){10}}
\multiput(20,50)(30,0){2}{\line(5,-3){10}}
\end{picture}
}
\put(255,180){ 
\begin{picture}(80,80)(0,0)
\put(47,30){\makebox(0,0){$+$}} \put(40,20){\line(0,1){30}}
\put(10,50){\line(1,0){60}} \put(20,50){\line(5,3){10}}
\put(20,50){\line(5,-3){10}} \put(60,45){\makebox(0,0){$-$}}
\end{picture}
}
\put(255,110){ 
\begin{picture}(80,80)(0,0)
\put(47,30){\makebox(0,0){$+$}} \put(40,20){\line(0,1){30}}
\put(10,50){\line(1,0){60}} \put(50,50){\line(5,3){10}}
\put(50,50){\line(5,-3){10}} \put(20,45){\makebox(0,0){$-$}}
\end{picture}
}
\put(255,40){ 
\begin{picture}(80,80)(0,0)
\put(40,35){\line(3,-5){6}} \put(40,35){\line(-3,-5){6}}
\put(40,20){\line(0,1){30}} \put(10,50){\line(1,0){60}}
\put(20,50){\line(5,3){10}} \put(20,50){\line(5,-3){10}}
\put(60,45){\makebox(0,0){$-$}}
\end{picture}
}
\put(255,-30){ 
\begin{picture}(80,80)(0,0)
\put(40,30){\line(3,5){6}} \put(40,30){\line(-3,5){6}}
\put(40,20){\line(0,1){30}} \put(10,50){\line(1,0){60}}
\put(50,50){\line(5,3){10}} \put(50,50){\line(5,-3){10}}
\put(20,45){\makebox(0,0){$-$}}
\end{picture}
}
\end{picture}
\caption{The catalogue of labelled junctions that are projections of trihedral vertices.} \label{fig:catalogue}
\end{figure}
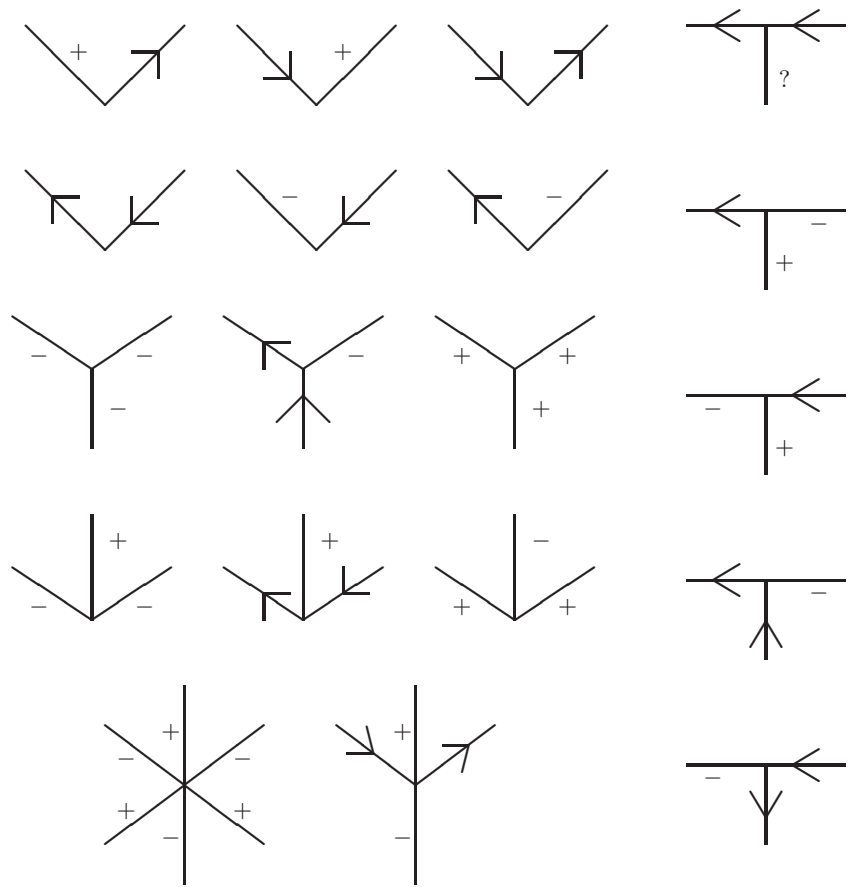

\thicklines
\setlength{\unitlength}{1pt}
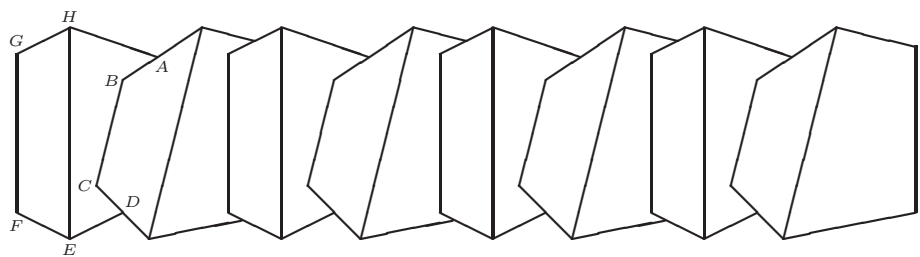
\begin{figure}
\centering
\begin{picture}(340,100)(0,-10)

\multiput(0,10)(80,0){4}{\line(0,1){60}} \multiput(20,0)(80,0){4}{\line(0,1){80}}
\multiput(30,20)(80,0){4}{\line(1,4){10}} \multiput(50,0)(80,0){4}{\line(1,4){20}}
\multiput(0,10)(80,0){4}{\line(2,-1){20}} \multiput(0,70)(80,0){4}{\line(2,1){20}}
\multiput(20,0)(80,0){4}{\line(2,1){20}} \multiput(20,80)(80,0){4}{\line(3,-1){33}}
\multiput(30,20)(80,0){4}{\line(1,-1){20}} \multiput(40,60)(80,0){4}{\line(3,2){30}}
\multiput(50,0)(80,0){3}{\line(5,1){35}} \multiput(70,80)(80,0){3}{\line(4,-1){20}}
\put(290,0){\line(5,1){50}} \put(310,80){\line(4,-1){30}} \put(340,10){\line(0,1){63}}

\put(55,65){\makebox(0,0){$_A$}}  \put(36,60){\makebox(0,0){$_B$}}
\put(26,20){\makebox(0,0){$_C$}} \put(44,14){\makebox(0,0){$_D$}} 
\put(20,-4){\makebox(0,0){$_E$}}  \put(0,5){\makebox(0,0){$_F$}}
\put(0,75){\makebox(0,0){$_G$}}  \put(20,84){\makebox(0,0){$_H$}}

\end{picture}
\caption{An example from a family of line drawings whose exponential number of labellings is reduced to one
by snake substitution.} \label{fig:LDblocks}
\end{figure}

To give a non-numerical example, we considered the impact of SS and CNS in the classic
problem of labelling line-drawings of polyhedral scenes composed of objects with
trihedral vertices~\cite{DBLP:journals/ai/Clowes71,Huff1,Waltz}. 
There are six types of trihedral vertices: $A$, $B$, $C$, $D$, $E$ and $F$, shown in Figure~\ref{fig:vertices}.
The aim is to assign each line in the drawing a semantic label among four possibilities: convex ($+$), concave ($-$) or occluding 
($\leftarrow$ or $\rightarrow$ depending whether the occluding surface is above or below the line).
Some lines in the top middle drawing in Figure~\ref{fig:vertices} have been labelled 
to illustrate the meaning of these labels. This problem can be expressed as a binary CSP by treating the junctions as variables.
The domains of variables are given by the catalogue of physically realisable labellings of the corresponding junction 
according to its type. This catalogue of junction labellings is obtained by considering the six vertex types viewed from all possible 
viewpoints~\cite{DBLP:journals/ai/Clowes71,Huff1}. For example,
there are 6 possible labellings of an L-junction, 8 for a T-junction, 5 for a Y-junction and 3 for a W-junction~\cite{ldbook}.
The complete catalogue of labelled junctions is shown in Figure~\ref{fig:catalogue}, where a question
mark represents any of the four labels and rotationally symmetric labellings are omitted.
There is a constraint between any two junctions joined by a line: this line must have the same semantic label
at both ends. We can also apply binary constraints between distant junctions: the 2Reg constraint limits 
the possible labellings of junctions such as $A$ and $D$ in Figure~\ref{fig:LDblocks}, since two non-colinear lines,
such as $AB$ and $CD$, which separate the same two regions cannot both be concave~\cite{lddl,ldbook}.

\thicklines
\setlength{\unitlength}{1pt}
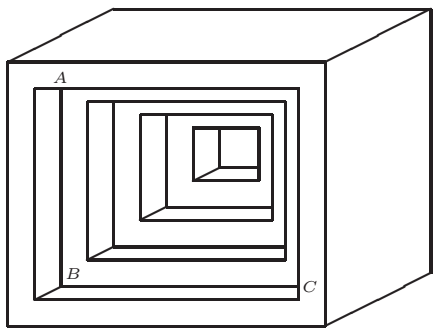
\begin{figure}
\centering
\begin{picture}(160,120)(0,0)

\put(0,0){\line(1,0){120}} \put(10,10){\line(1,0){100}} \put(20,15){\line(1,0){90}}
\put(30,25){\line(1,0){75}} \put(40,30){\line(1,0){65}} \put(50,40){\line(1,0){50}}
\put(60,45){\line(1,0){40}} \put(70,55){\line(1,0){25}} \put(80,60){\line(1,0){15}}
\put(70,75){\line(1,0){25}} \put(50,80){\line(1,0){50}} \put(30,85){\line(1,0){75}}
\put(10,90){\line(1,0){100}} \put(0,100){\line(1,0){120}} \put(40,120){\line(1,0){120}}

\put(0,0){\line(0,1){100}} \put(10,10){\line(0,1){80}} \put(20,15){\line(0,1){75}}
\put(30,25){\line(0,1){60}} \put(40,30){\line(0,1){55}} \put(50,40){\line(0,1){40}}
\put(60,45){\line(0,1){35}} \put(70,55){\line(0,1){20}} \put(80,60){\line(0,1){15}}
\put(95,55){\line(0,1){20}} \put(100,40){\line(0,1){40}} \put(105,25){\line(0,1){60}}
\put(110,10){\line(0,1){80}} \put(120,0){\line(0,1){100}} \put(160,20){\line(0,1){100}}

\put(10,10){\line(2,1){10}} \put(30,25){\line(2,1){10}} \put(50,40){\line(2,1){10}}
\put(70,55){\line(2,1){10}} \put(0,100){\line(2,1){40}} \put(120,100){\line(2,1){40}}
\put(120,0){\line(2,1){40}}

\put(20,94){\makebox(0,0){$_A$}}  \put(25,20){\makebox(0,0){$_B$}}
\put(114.5,15){\makebox(0,0){$_C$}} 

\end{picture}
\caption{An example from a family of line drawings whose exponential number of labellings is reduced to one
by snake substitution.} \label{fig:LDcube}
\end{figure}

The drawings shown in Figure~\ref{fig:LDblocks} and Figure~\ref{fig:LDcube} are
ambiguous. For example, in Figure~\ref{fig:LDblocks}, 
any of lines $AB$, $BC$ or $CD$ could be projections of concave edges (meaning that the two blocks on
the left side of the figure are part of the same object) or all three could be
projections of occluding edges (meaning that these two blocks are, in fact, separate objects).
Similarly, lines $AB$ and $BC$ in Figure~\ref{fig:LDcube} could be projections of occluding or concave edges. 
The drawings shown in Figure~\ref{fig:LDblocks} and Figure~\ref{fig:LDcube} are both examples of 
families of line drawings. In each of these figures there are four copies of the basic structure, but there is a clear
generalisation to drawings containing $n$ copies of the basic structure. The ambiguity that we have pointed
out above gives rise to an exponential number of valid labellings for these families of drawings. However,
after applying arc consistency and snake substitution until convergence, each domain is a singleton
for both these families of line drawings.
We illustrate this by giving one example of a snake substitution. After arc consistency has been established, 
the labelling $(-,+,-)$ for junction $E$ in Figure~\ref{fig:LDblocks} is snake substitutable by $(\leftarrow,+,\leftarrow)$.
This can be seen by consulting Figure~\ref{fig:LDcloseup}
which shows the domains of variables $G$, $F$, $E$, $D$ and $C$ with lines joining compatible
labellings for adjacent junctions: snake substitutability follows from the fact that the labelling $(-,+,-)$ for $E$ can be replaced by 
$(\leftarrow,+,\leftarrow)$ in any global labelling, provided the labelling $(\uparrow,-)$ for $F$ is also
replaced by $(\uparrow,\leftarrow)$ and the labelling $(\leftarrow,-,\leftarrow)$ for $D$ is also replaced
by $(\leftarrow,\leftarrow,\leftarrow)$. Purely for clarity of presentation, some constraints have not been shown
in Figure~\ref{fig:LDcloseup}, notably the 2Reg constraints between distant junctions~\cite{lddl,ldbook}.

Of course, there are line drawings where snake substution is much less effective than in Figures~\ref{fig:LDblocks}
and \ref{fig:LDcube}. Nevertheless, in the six drawings in Figure~\ref{fig:vertices}, which are a representative
sample of simple line drawings, 22 of the 73 junctions have their domains reduced to singletons by arc
consistency alone and a further 20 junctions have their domains reduced to singletons when both arc
consistency and snake substitution are applied. This can be compared with neighbourhood substitution
which eliminates no domain values in this sample of six drawings.
It should be mentioned that we found no examples where conditioned neighbourhood substitution  
led to the elimination of labellings in the line-drawing labelling problem.

\thicklines \setlength{\unitlength}{1pt}
\newsavebox{\lu}
\savebox{\lu}(20,10){
\begin{picture}(20,10)(0,0)
\put(4,8){\line(1,0){13}}  \put(4,8){\line(5,-6){9}}
\end{picture}
}
\newsavebox{\LU}
\savebox{\LU}(10,10){
\begin{picture}(10,10)(0,0)
\put(3.5,7.5){\line(1,0){8}}  \put(3.5,7.5){\line(0,-1){8}}
\end{picture}
}
\newsavebox{\LUminus}
\savebox{\LUminus}(10,10){
\begin{picture}(10,10)(0,0)
\put(0,3){\makebox(0,0){{ $-$}}}
\end{picture}
}
\newsavebox{\LD}
\savebox{\LD}(10,10){
\begin{picture}(10,10)(0,0)
\put(1,1){\line(1,0){8}}  \put(1,1){\line(0,1){8}}
\end{picture}
}
\newsavebox{\LDminus}
\savebox{\LDminus}(10,10){
\begin{picture}(10,10)(0,0)
\put(0,5){\makebox(0,0){$-$}}
\end{picture}
}
\newsavebox{\UP}
\savebox{\UP}(10,14){
\begin{picture}(10,14)(0,-4)
\put(0,6){\line(3,-5){6}}  \put(0,6){\line(-3,-5){6}}
\end{picture}
}
\newsavebox{\ovalthree}
\savebox{\ovalthree}(40,150){
\begin{picture}(40,150)(0,0)
\put(20,75){\oval(40,150)} \multiput(20,50)(0,40){3}{\makebox(0,0){{\large $\bullet$}}}
\end{picture}
}

\thicklines
\setlength{\unitlength}{1pt}
\begin{figure}
\centering
\begin{picture}(280,165)(0,0)

\put(10,33){\usebox{\ru}} \put(10,113){\usebox{\ru}}
\put(10,19){\usebox{\UP}} \put(10,59){\usebox{\UP}}  
\multiput(10,19)(0,40){3}{\line(0,1){14}} \multiput(10,33)(0,40){3}{\line(2,1){20}}
\put(21,73){\makebox{$-$}} \put(12,103){\makebox{$-$}} 

\put(74,60){\usebox{\lu}} \put(74,100){\usebox{\lu}} \put(74,30){\usebox{\UP}} \put(74,70){\usebox{\UP}}
\multiput(74,30)(0,40){3}{\line(0,1){14}} \multiput(74,30)(0,40){3}{\line(2,-1){20}}
\put(85,25){\makebox{$-$}} \put(76,115){\makebox{$-$}}

\put(120,83){\usebox{\lu}} \put(140,83){\usebox{\ld}}
\multiput(140,23)(0,60){2}{\line(0,1){20}} \multiput(140,23)(0,60){2}{\line(2,1){19}} \multiput(140,23)(0,60){2}{\line(-2,1){19}}
\put(120,24){\makebox{$-$}} \multiput(142,36)(0,60){2}{\makebox{$+$}} \put(153,24){\makebox{$-$}}

\put(190,103){\usebox{\LD}} \put(190,63){\usebox{\LD}} \put(190,23){\usebox{\LDminus}}
\put(200,23){\usebox{\LU}} \put(190,33){\usebox{\LU}} \put(200,63){\usebox{\LU}} \put(190,73){\usebox{\LU}}
\put(200,103){\usebox{\LU}} \put(190,113){\usebox{\LUminus}} 
\multiput(190,23)(0,40){3}{\line(1,1){10}} \multiput(190,43)(0,40){3}{\line(1,-1){20}}

\put(254,20){\usebox{\lu}} \put(254,60){\usebox{\lu}}
\put(254,30){\usebox{\UP}} \put(254,110){\usebox{\UP}}
\multiput(254,30)(0,40){3}{\line(0,1){14}} \multiput(254,30)(0,40){3}{\line(2,-1){20}}
\put(255,73){\makebox{$-$}} \put(266,105){\makebox{$-$}}

\multiput(0,0)(60,0){2}{\usebox{\ovalthree}}  \multiput(180,0)(60,0){2}{\usebox{\ovalthree}} 
\put(140,75){\oval(45,150)} 
\multiput(140,50)(0,60){2}{\makebox(0,0){{\large $\bullet$}}}
\put(20,50){\line(1,0){240}} \put(20,90){\line(1,0){60}} \put(20,130){\line(1,0){60}}
\put(200,90){\line(1,0){60}} \put(200,130){\line(1,0){60}}  \put(80,90){\line(3,1){120}}  
\put(20,50){\line(3,2){60}} \put(80,130){\line(3,-1){120}}  \put(200,90){\line(3,-2){60}}
\put(200,50){\line(3,2){60}}
\put(20,157){\makebox(0,0){$G$}} \put(80,157){\makebox(0,0){$F$}}
\put(140,157){\makebox(0,0){$E$}} \put(200,157){\makebox(0,0){$D$}} \put(260,157){\makebox(0,0){$C$}}

\end{picture}
\caption{A close-up view of the variables corresponding to junctions $G$, $F$, $E$, $D$ and $C$
in Figure~\ref{fig:LDblocks}.} \label{fig:LDcloseup}
\end{figure}
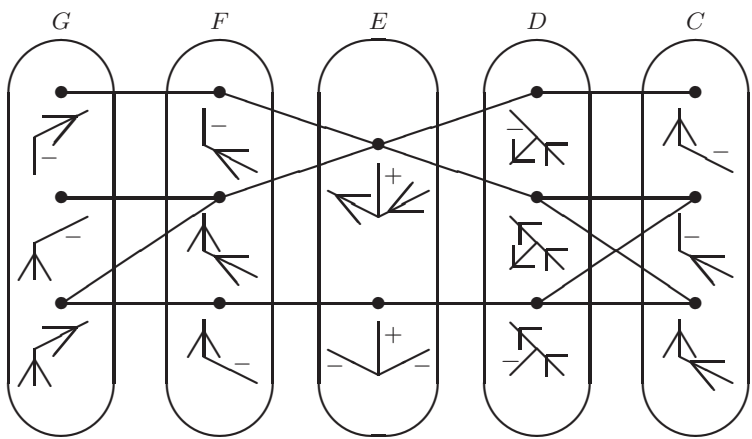

\section{Complexity}  \label{sec:complexity}

In a binary CSP instance $(X,\mathcal{D},R)$,  
we say that two variables $x_i, x_j \in X$ constrain each other if there is a non-trivial constraint between them
(i.e. $R_{ij} \neq \mathcal{D}(x_i) \times \mathcal{D}(x_j)$). Let 
$E \subseteq \{1,\ldots,n\} \times \{1,\ldots,n\}$ denote the set of pairs $\{i,j\}$ such that $x_i,x_j$ constrain each other.
We use $d$ to denote the maximum size of the domains $\mathcal{D}(x_i)$
and $e = |E|$ to denote the number of non-trivial binary constraints.
In this section we show that it is possible to apply CNS and SS until convergence in $O(ed^3)$ time and that it
is possible to check SCSS in $O(ed^3)$ time.
Thus, the complexity of applying the value-elimination rules CNS, SS and SCSS is comparable to the $O(ed^3)$ 
time complexity of applying neighbourhood substitution (NS)~\cite{ns}. 
This is interesting because (in instances with more than one variable)
CNS, SS and SCSS all strictly subsume NS.

\subsection{Substitution and arc consistency}

It is well known that arc consistency eliminations can provoke new eliminations by neighbourhood
substitution (NS) but that NS eliminations cannot destroy arc consistency~\cite{ns}.
It follows that arc consistency eliminations can provoke new eliminations by SS, CNS and SCSS
(since these notions subsume NS). It is easily seen from Definition~\ref{def:CNS} that 
eliminations by CNS cannot destroy arc consistency. We therefore assume in this section
that arc consistency has been established before looking for eliminations by any form of substitution.
Nonetheless, unlike CNS, eliminations by SS (or SCSS) can provoke new eliminations by arc consistency;
however, these eliminations cannot themselves propagate. To see this, suppose that $b \in\mathcal{D}(x_i)$
is eliminated since it is snake-substitutable by $a$. If $b$ is the only support of $d \in \mathcal{D}(x_k)$
at $x_i$, then $d$ can then be eliminated by arc consistency. However, the elimination of $d$ cannot
provoke any new eliminations by arc consistency. To see this, recall that, by Definition~\ref{def:SS} of SS, 
there is a value $e \in \mathcal{D}(x_k)$ such that for all $\ell \neq i,k$, for all $f\in \mathcal{D}(x_{\ell})$,
if $d$ was a support for $f$ at $x_k$ then so was $e$ (as illustrated in Figure~\ref{fig:snake}). 
Furthermore, since $b$ was the only support for $d$
at $x_i$, no other value in $\mathcal{D}(x_i)$ can lose its  support when $d$ is eliminated from $\mathcal{D}(x_k)$.
In conclusion, the algorithm for applying SS has to apply this limited form of arc-consistency
(without propagation) whereas the algorithm to apply CNS does not need to test for arc consistency
since we assume that it has already been established. Furthermore, since AC is, in fact, subsumed by SCSS we do
not explicitly need to test for it in the algorithm to apply SCSS.
 
\subsection{Applying SS until convergence}  \label{subsec:SS}

We first give an algorithm for applying SS until convergence. 
The following notions used by the algorithm are best understood by consulting Figure~\ref{fig:snake}.
An assignment $f \in \mathcal{D}(x_{\ell})$ that contradicts some $d \xrightarrow{\scriptstyle{k \ell}} e$
is known as a \emph{block}, and the associated variable $x_{\ell}$ is known as a \emph{block variable}. 
For $a \in \mathcal{D}(x_i)$, the term \emph{sub} denotes a value $e \in \mathcal{D}(x_k)$ that
could replace $d$ when $x_i$ is assigned $a$ (in the sense that 
$(a,e) \in R_{ik}$ and $\forall \ell \notin \{i,k\}, d \xrightarrow{\scriptstyle{k \ell}} e$, as
illustrated in Figure~\ref{fig:snake}). 
In the context of a possible snake substitution of $b \in \mathcal{D}(x_i)$ by $a$, a value $d \in \mathcal{D}(x_k)$ 
that does not have a sub that could replace it is 
a \emph{stop} and the corresponding variable $x_k$ is a \emph{stop variable}. 
If there are no stop variables, then $b$ is snake substitutable by $a$.
The algorithm for applying SS until convergence uses the following data structures:
\begin{itemize}
\item For all $\{k,\ell\} \in E$, for all $d,e \in \mathcal{D}(x_k)$, \\
NbBlocks($k,d,e,\ell$) $=$ $|\{ f \in \mathcal{D}(x_{\ell}) \mid (d,f) \in R_{k \ell} \land (e,f) \notin R_{k \ell} \}|$.
\item For all $k \in \{1,\ldots,n\}$, for all $d,e \in \mathcal{D}(x_k)$, \\
BlockVars($k,d,e$) $=$ $\{ \ell \mid \{k,\ell\} \in E \land \text{NbBlocks}(k,d,e,\ell) > 0 \}$. \\ If 
BlockVars($k,d,e)$ $=$ $\emptyset$, for some $e \neq d$, then $d$ can be eliminated from $\mathcal{D}(x_k)$ by
neighbourhood substitution.
\item For all $\{i,k\} \in E$, for all $a \in \mathcal{D}(x_i)$, for all $d \in \mathcal{D}(x_k)$
such that $(a,d) \notin R_{ik}$, \\  NbSubs($i,a,k,d$) $=$ $| \{ e \in \mathcal{D}(x_k) \mid (a,e) \in R_{ik}
\land \text{BlockVars}(k,d,e) \subseteq \{i\} \} |$. \\
Note that $\text{BlockVars}(k,d,e) \subseteq \{i\} $ if and only if $\forall \ell \notin \{i,k\}, d \xrightarrow{\scriptstyle{k \ell}} e$.
\item For all $\{i,k\} \in E$, for all $a,b \in \mathcal{D}(x_i)$, 
NbStops($i,a,b,k$) $=$ \\ $| \{ d \in \mathcal{D}(x_k) \mid 
(b,d) \in R_{ik} \land (a,d) \notin R_{ik} \land \text{NbSubs}(i,a,k,d) = 0 \} |$.
\item For all $i \in \{1,\ldots,n\}$, for all $a,b \in \mathcal{D}(x_i)$, \\
NbStopVars($i,a,b$) $=$ $| \{ k \mid \{i,k\} \in E \land \text{NbStops}(i,a,b,k) \neq 0 \} |$.
\item  For all $i \in \{1,\ldots,n\}$, for all $b \in \mathcal{D}(x_i)$, \\
NbSnake($i,b$) $=$ $| \{ a \in \mathcal{D}(x_i) \mid \text{NbStopVars}(i,a,b) = 0 \} |$. \\
If NbSnake($i,b$) $> 0$, then value $b$ can be eliminated from $\mathcal{D}(x_i)$ by snake substitution.
\item  For all $i \in \{1,\ldots,n\}$, for all $b \in \mathcal{D}(x_i)$, \\
Inconsistent($i,b$) $=$ true if $b$ has no support at some other variable. 
\end{itemize}

We assume, throughout this section, that each set (such as BlockVars($k,d,e$)) which is a subset
of a fixed set $S$ (in this case the set of $n$ variables) is stored using
an array data structure whose index ranges over 
the elements of $S$ thus allowing all basic operations, such as insertion and deletion, 
to be performed in $O(1)$ time.
The data structures listed above can clearly be initialised in $O(ed^3)$ time and require $O(ed^2)$ space.
We also use a list data structure ElimList to which we add pairs $\langle i,b \rangle$ if 
NbSnake($i,b$) $> 0$ (indicating that $b$ could be eliminated by snake substitition). 
Note that due to possible eliminations of other values, from the same
or other domains, between the detection of the snake substitutability of $b$ and the moment this
information is processed, $b$ may no longer be snake substitutable when $\langle i,b \rangle \in$ ElimList
is processed. We choose to give precedence to neighbourhood substitutability over snake substitutability
by placing neighbourhood-substitutable values at the head of the list ElimList (line (1) in the code below)
and snake-substitutable values at the tail. We also add to ElimList values that can be eliminated by arc
consistency.

The processing of elements of ElimList involves the updating of all the above data structures
which can in turn lead to the detection of new eliminations. 
The algorithm given in Figure~\ref{fig:SSalgo} performs eliminations
and propagates until convergence, assuming that the above data structures (including ElimList) have all been initialised.
When a value $u$ is eliminated from $\mathcal{D}(x_r)$ we have to update NbBlocks 
because $u$ may correspond to value $f$ in Figure~\ref{fig:snake}, NbSubs
because $u$ may correspond to value $e$ in Figure~\ref{fig:snake}, NbStops
because $u$ may correspond to the value $d$ in Figure~\ref{fig:snake}, NbSnake
because $u$ may correspond to the value $a$ in Figure~\ref{fig:snake}, and Inconsistent
when $u$ corresponds to value $b$ in Figure~\ref{fig:snake}.

\begin{figure} \centering
\fbox{\parbox{12cm}{ 
\begin{tabbing}
\hspace{4mm} \= \hspace{4mm} \= \hspace{4mm} \= \hspace{4mm} \= \hspace{4mm} \= \hspace{4mm} \= 
\hspace{4mm} \= \hspace{4mm} \= \hspace{4mm} \= \hspace{4mm} \= \hspace{-47mm} 
while ElimList $\neq \emptyset$ : \\
\> pop $\langle r,u \rangle$ from ElimList ; \\
\> if $u \in \mathcal{D}(x_r)$ and (NbSnake($r,u$) $> 0$  or Inconsistent($r,u$)) : \\
\> \> delete $u$ from $\mathcal{D}(x_r)$ ; \\
\> \> *** Update NbBlocks and propagate *** \\
\> \> for all $k$ such that $\{k,r\} \in E$ : \\
\> \> \> for all $d,e \in \mathcal{D}(x_k)$ such that $(d,u) \in R_{kr}$ and $(e,u) \notin R_{kr}$ : \\
\> \> \> \> NbBlocks($k,d,e,r$) := NbBlocks($k,d,e,r$) $- 1$ ; \\
\> \> \> \> if NbBlocks($k,d,e,r$) $= 0$  then \\
\> \> \> \> \> delete $r$ from BlockVars($k,d,e$) ; \\
\> \> \> \> \> if BlockVars($k,d,e$) becomes $\emptyset$ then  \ \ *** NS *** \\
\> \> \> \> \> \> add $\langle k,d \rangle$ to head of ElimList ; .......................................(1) \\
\> \> \> \> \> if BlockVars($k,d,e$) becomes a singleton $\{i\}$ then \\
\> \> \> \> \> \> for all $a \in \mathcal{D}(x_i)$ such that $(a,d) \notin R_{ik}$ and $(a,e) \in R_{ik}$ : \\
\> \> \> \> \> \> \> NbSubs($i,a,k,d$) := NbSubs($i,a,k,d$) $+ 1$ ; ...........(2) \\
\> \> \> \> \> \> \> if NbSubs($i,a,k,d$) $= 1$ then DecStops($i,a,b,k$) ; \\
\> \> for all $i$ such that $\{i,r\} \in E$ : \\
\> \> \> *** Update NbSubs and propagate *** \\
\> \> \> for all $(a,d) \in \mathcal{D}(x_i) \times \mathcal{D}(x_r)$ such that $(a,d) \notin R_{ir}$ : \\
\> \> \> \> if $(a,u) \in R_{ir}$ and BlockVars($r,d,u$) $\subseteq \{i\}$ then .....................(3) \\
\> \> \> \> \> NbSubs($i,a,r,d$) := NbSubs($i,a,r,d$) $- 1$ ; \\
\> \> \> \> \> if NbSubs($i,a,r,d$) $= 0$ then \\
\> \> \> \> \> \> for all $b \in \mathcal{D}(x_i)$ such that $(b,d) \in R_{ir}$ : \\
\> \> \> \> \> \> \> IncStops($i,a,b,r$) ; \\
\> \> \> *** Update NbStops and propagate *** \\
\> \> \> for all $a,b \in \mathcal{D}(x_i)$ such that $(b,d) \in R_{ir}$ and $(a,d) \notin R_{ir}$ \\
\> \> \> \> and NBSubs($i,a,r,d$) $= 0$ :    \\   
\> \> \> \> \> DecStops($i,a,b,r$) ; \  .............................................................(4) \\
\> \> \> *** Update Inconsistent *** \\
\> \> \> for all $v \in \mathcal{D}(x_i)$ : \\
\> \> \> \> if $v$ has no support at $x_r$ then \ \ *** not AC *** \\
\> \> \> \> \> Inconsistent($i,v$) := true ; \ \ add $(i,v)$ to head of ElimList ; \\
\> \> *** Update NbSnake *** \\
\> \> for all $b \in \mathcal{D}(x_r) \setminus \{u\}$ such that NbStopVars($r,u,b$) $= 0$ : \\
\> \> \> NbSnake($r,b$) := NbSnake($r,b$) $- 1$ ;
\end{tabbing}
}}
\caption{The propagation algorithm for applying SS until convergence.}  \label{fig:SSalgo}
\end{figure}
The subprograms DecStops and IncStops for updating NbStops (and the consequent updating
of NbStopVars and NbSnake) are given in Figure~\ref{fig:DecInc}.

\begin{figure} \centering
\fbox{\parbox{12cm}{
\begin{tabbing}
\hspace{4mm} \= \hspace{4mm} \= \hspace{4mm} \= \hspace{4mm} \= \hspace{4mm} \= \hspace{4mm} \= 
\hspace{4mm} \= \hspace{4mm} \= \hspace{4mm} \= \hspace{4mm} \= \hspace{-47mm} 
procedure DecStops($i,a,b,r$) : \\
\> NbStops($i,a,b,r$) := NbStops($i,a,b,r$) $- 1$ ; \\
\> if NbStops($i,a,b,r$) $= 0$ then \\
\> \> NbStopVars($i,a,b$) := NbStopVars($i,a,b$) $- 1$ ; \\
\> \> if NbStopVars($i,a,b$) $= 0$ then \\
\> \> \> NbSnake($i,b$) := NbSnake($i,b$) $+ 1$ ; \\
\> \> \> if NbSnake($i,b$) $= 1$ then \ \ \ *** NbSnake($i,b$) becomes non-zero *** \\
\> \> \> \> add $\langle i,b \rangle$ to the end of ElimList ; \ \ *** SS *** 
\end{tabbing}
\begin{tabbing}
\hspace{4mm} \= \hspace{4mm} \= \hspace{4mm} \= \hspace{4mm} \= \hspace{4mm} \= \hspace{4mm} \= 
\hspace{4mm} \= \hspace{4mm} \= \hspace{4mm} \= \hspace{4mm} \= \hspace{-47mm} 
procedure IncStops($i,a,b,r$) : \\
\> NbStops($i,a,b,r$) := NbStops($i,a,b,r$) $+ 1$ ; \\
\> if NbStops($i,a,b,r$) $= 1$ then \ \ \ *** NbStops($i,a,b,r$) becomes non-0 *** \\
\> \> NbStopVars($i,a,b$) := NbStopVars($i,a,b$) $+ 1$ ; \\
\> \> if NbStopVars($i,a,b$) $= 1$ then \ \ \ *** NbStopVars($i,a,b$) becomes non-0 *** \\
\> \> \> NbSnake($i,b$) := NbSnake($i,b$) $- 1$ ; 
\end{tabbing}
}}
\caption{Subprograms DecStops and IncStops used when applying SS until convergence.}  \label{fig:DecInc}
\end{figure}

We now analyse the complexity of the propagation algorithm (Figure~\ref{fig:SSalgo}). 
First, observe that the subprograms DecStops and IncStops
both have $O(1)$ time complexity.
Both NbBlocks and BlockVars are monotone decreasing, so the total number of updates to these two data
structures is clearly $O(ed^3)$. For fixed $k,d,e$, since BlockVars($k,d,e$) is monotone, it can become equal to a singleton
at most once. This implies that the total number of times line (2) will be executed is $O(nd^3)$.
Since there can be at most $d$ eliminations from each domain, lines (3) and (4) can be executed at most $(ed^3)$ times.
It follows that the time complexity of this algorithm to apply SS until convergence
(including the initialisation and propagation steps) is $O(ed^3)$. 
We state formally in the following theorem what we have proved in this section.

\begin{theorem}
Value eliminations by snake substitution can be applied until convergence in $O(ed^3)$ time and $O(ed^2)$ space.
\end{theorem}

\subsection{Applying CNS until convergence}

Before giving an algorithm for performing value-eliminations by CNS until convergence, we
first consider the interaction between neighbourhood substitution and CNS.
Recall that CNS subsumes neighbourhood substitution. It is also clear from Definition~\ref{def:CNS} of CNS
that eliminating values by neighbourhood substitution cannot prevent elimination of other values by CNS.
However, the converse is not true: eliminations by CNS can prevent eliminations 
of other values by NS. To see this, 
consider a 2-variable instance with constraint $(x_1 = x_2) \vee (x_2=0)$ and domains
$\mathcal{D}(x_1) = \{1,\ldots,d-1\}$, $\mathcal{D}(x_2) = \{0,\ldots,d-1\}$. The value $0 \in \mathcal{D}(x_2)$
can be eliminated by CNS (conditioned by the variable $x_1$) since $\forall c \in \mathcal{D}(x_1)$,
$\exists a = c \in \mathcal{D}(x_2) \setminus \{0\}$ such that $(a,c) \in R_{12}$.
After eliminating $0$ from $\mathcal{D}(x_2)$, no further eliminations are possible by CNS or neighbourhood substitution.
However, in the original instance we could have eliminated all elements of $\mathcal{D}(x_2)$ except $0$
by neighbourhood substittion. Thus, in our algorithm to apply CNS, we give priority to eliminations by NS.

In the context of a substitution of $b$ at $x_i$ conditioned by $x_j$ (see Figure~\ref{fig:CNS}), we say that
$a \in \mathcal{D}(x_i) \! \setminus \! \{b\}$ is a \emph{cover} for $c \in \mathcal{D}(x_j)$ if
$(a,c) \in R_{ij}$ and $\forall k \notin \{i,j\},  b \xrightarrow{\scriptstyle{ik}} a$. The algorithm for applying CNS until convergence 
requires the data structures NbBlocks and BlockVars, as described in Section~\ref{subsec:SS}, 
along with the following data structures:
\begin{itemize}
\item For all $\{i,j\} \in E$, for all $b \in \mathcal{D}(x_i)$, for all $c \in \mathcal{D}(x_j)$, \ \
NbCovers($i,b,j,c$) $=$ \\ $| \{ a \in \mathcal{D}(x_i) \! \setminus \! \{b\}
 \mid (a,c) \in R_{ij} \land \text{BlockVars}(i,b,a) \subseteq \{j\} \} |$.
\item For all $\{i,j\} \in E$, for all $b \in \mathcal{D}(x_i)$, \\
Uncovered($i,b,j$) $=$ $\{ c \in D(x_j) \mid (b,c) \in R_{ij} \land \text{NbCovers}(i,b,j,c) = 0 \}$. \\
If Uncovered($i,b,j$) $= \emptyset$ then $b$ can be eliminated from $\mathcal{D}(x_i)$ by CNS
conditioned by variable $x_j$.
\end{itemize}
The algorithm uses two separate lists for possible eliminations:
\begin{itemize}
\item NSlist is a list of triples $(i,b,a)$ such that $b$ can be eliminated from $\mathcal{D}(x_i)$ by
neighbourhood substitution by $a$, provided neither $a$ nor $b$ have already been eliminated from $\mathcal{D}(x_i)$.
\item CNSlist is a list of tuples $(i,b,j)$ such that $b$ can be elminated from $\mathcal{D}(x_i)$ by
CNS conditioned by variable $x_j$, provided Uncovered($i,b,j$) is still empty.
\end{itemize}
It is easy to see that these data structures require $O(ed^2)$ space and can be initialised in
$O(ed^3)$ time.

The algorithm given in Figure~\ref{fig:CNSalgo} processes elements on the two lists 
NSlist and CNSlist, and propagates eliminations until convergence, assuming that the above data structures
have been initialised after first establishing arc consistency. 
Recall that eliminations by CNS cannot destroy arc consistency.
When $u$ is eliminated from $\mathcal{D}(x_r)$, we have
to update NbBlocks because $u$ may correspond to $d$ in Figure~\ref{fig:CNS}, update NbCovers because $u$
may correspond to $a$ in Figure~\ref{fig:CNS}, and update Uncovered because $u$ may correspond to
$c$ in Figure~\ref{fig:CNS}.

\begin{figure} \centering
\fbox{\parbox{12cm}{
\begin{tabbing}
\hspace{4mm} \= \hspace{4mm} \= \hspace{4mm} \= \hspace{4mm} \= \hspace{4mm} \= \hspace{4mm} \= 
\hspace{4mm} \= \hspace{4mm} \= \hspace{4mm} \= \hspace{4mm} \= \hspace{-47mm} 
while (NSlist $\neq \emptyset$) or (CNSlist  $\neq \emptyset$)  : \\
\> if NSlist $\neq \emptyset$ then \\
\> \> pop $(p,u,v)$ from NSlist ; \\
\> \> OK := \ ($u,v \in \mathcal{D}(x_p)$) ; \\
\> else \\
\> \> pop $(p,u,q)$ from CNSlist ; \\
\> \> OK := \ ($u \in \mathcal{D}(x_p) \ \land \ \text{Uncovered}(p,u,q))$ ; \\
\> if OK then \\
\> \> delete $u$ fom $\mathcal{D}(x_p)$ ; \\
\> \> *** Update NbBlocks and propagate *** \\
\> \> for all $i$ such that $\{i,p\} \in E$ : \\
\> \> \> for all $a,b \in \mathcal{D}(x_i)$ such that $(b,u) \in R_{ip}$ and $(a,u) \notin R_{ip}$ : ............(1) \\
\> \> \> \> NbBlocks($i,b,a,p$) := NbBlocks($i,b,a,p$) $- 1$ ; \\
\> \> \> \> if NbBlocks($i,b,a,p$) $= 0$ then \\
\> \> \> \> \> BlockVars($i,b,a$) := BlockVars($i,b,a$) $\setminus \{p\}$ ; \\
\> \> \> \> \> if BlockVars($i,b,a$) becomes $\emptyset$ then \ \ \ *** NS ***  \\
\> \> \> \> \> \> add $(i,b,a)$ to NSlist ; \\
\> \> \> \> \> if BlockVars($i,b,a$) becomes a singleton $\{j\}$ then ....................(2) \\
\> \> \> \> \> \> for all $c \in \mathcal{D}(x_j)$ :  ...............................................................(3) \\ 
\> \> \> \> \> \> \> if $(a,c) \in R_{ij}$ then \\
\> \> \> \> \> \> \> \> NbCovers($i,b,j,c$) := NbCovers($i,b,j,c$) $+ 1$ ; \\
\> \> \> \> \> \> \> \> Uncovered($i,b,j$) := Uncovered($i,b,j$) $\setminus \{c\}$ ; \\
\> \> \> \> \> \> \> \> if Uncovered($i,b,j$) becomes $\emptyset$ then \ \ \ *** CNS *** \\
\> \> \> \> \> \> \> \> \> add $(i,b,j)$ to CNSlist ; \\
\> \> *** Update NbCovers and propagate *** \\
\> \> for all $b \in \mathcal{D}(x_p) \setminus \{u\}$ : \\
\> \> \> for all $j$ such that $\{j,p\} \in E$ : \\
\> \> \> \> for all $c \in \mathcal{D}(x_j)$ :  ............................................................................(4) \\
\> \> \> \> \> if $(u,c) \in R_{pj}$ and BlockVars($p,b,u$) $\subseteq \{j\}$ then \\
\> \> \> \> \> \> NbCovers($p,b,j,c$) := NbCovers($p,b,j,c$) $- 1$ ; \\
\> \> \> \> \> \> if NbCovers($p,b,j,c$) $= 0$ then \\
\> \> \> \> \> \> \> add $c$ to Uncovered($p,b,j$) ; \\
\> \> *** Update  Uncovered *** \\
\> \> for all $i$ such that $\{i,p\} \in E$ :\\
\> \> \> for all $b \in \mathcal{D}(x_i)$ : ..................................................................................(5) \\
\> \> \> \> Uncovered($i,b,p$) := Uncovered($i,b,p$) $\setminus \{u\}$ ; \\
\> \> \> \> if Uncovered($i,b,p$) becomes $\emptyset$ then \ \ \ *** CNS *** \\
\> \> \> \> \> add $(i,b,p)$ to CNSlist ; 
\end{tabbing}
}}
\caption{The propagation algorithm for applying CNS until convergence.}  \label{fig:CNSalgo}
\end{figure}

The time complexity of this algorithm is clearly determined by the complexity of the loops at lines (1), (3), (4) and (5).
Since there are at most $d$ eliminations from each domain, we can easily see that lines (1) and (4) are executed
$O(ed^3)$ times, whereas line (5) is executed $O(ed^2)$ times.
Since each BlockVars($i,a,b$) is monotone decreasing, it can become equal to a singleton at most once. It follows
that line (2) is executed $O(nd^2)$ times and hence that line (3) is executed $O(nd^3)$ times.
We can conclude that that initalisation and propagation steps both require $O(ed^3)$ time
and $O(ed^2)$ space. We state formally in the following theorem what we have just proved.

\begin{theorem}
Value eliminations by conditioned neighbourhood substitution can be applied 
until convergence in $O(ed^3)$ time and $O(ed^2)$ space.
\end{theorem}

\subsection{Applying SCSS until convergence}

In this subsection we study snake-conditioned snake substitution (SCSS). We show that
it is possible to check whether some SCSS is possible in a binary CSP instance in $O(ed^3)$ time and 
$O(ed^2)$ space. Although we do not have the same bound for applying SCSS value-eliminations
until convergence, from a practical point of view this is not necessarily an important point.
It is well known in the constraint programming community that non-optimised algorithms
such as AC-3~\cite{DBLP:journals/ai/Mackworth77} are in practice no slower on average 
than their optimised counterparts, such as AC-2001~\cite{ac}. The upper bound we give on the
time complexity of applying SCSS until convergence of $O(end^5)$ is no doubt very
pessimistic and will probably only be attained in pathological examples.

Among the data structures that allow us to check for SCSS eliminations, we require
NbBlocks, BlockVars, NbSubs and NbStops, as defined in Section~\ref{subsec:SS}. We also require
the following data structures (where a \emph{snake cover} for $c \in \mathcal{D}(x_j)$ corresponds
to the value $a \in \mathcal{D}(x_i)$ of Figure~\ref{fig:SCSS} and Definition~\ref{def:SCSS}):
\begin{itemize}
\item For all $i \in \{1,\ldots,n\}$, for all $a,b \in \mathcal{D}(x_i)$, \\
StopVars($i,a,b$) $=$ $\{ k \mid \{i,k\} \in E \land \text{NbStops}(i,a,b,k) > 0 \}$.
\item For all $\{i,j\} \in E$, for all $b \in \mathcal{D}(x_i)$, for all $c \in \mathcal{D}(x_j)$, \\ 
NbSnakeCovers($i,b,j,c$)  $=$ $| \{ a \! \in \! \mathcal{D}(x_i) \! \setminus \! \{b\} \mid ((a,c) \! \in \! R_{ij}
 \lor \text{NbSubs}(i,a,j,c) > 0) \land \text{StopVars}(i,a,b) \subseteq \{j\} \} |$.
\item For all $\{i,j\} \in E$, for all $b \in \mathcal{D}(x_i)$, \
NotSnakeCovered($i,b,j$) $=$ \\ 
$\{ c \in \mathcal{D}(x_j) \mid (b,c) \in R_{ij} \land \text{NbSnakeCovers}(i,b,j,c) = 0 \}$.
\end{itemize}
It follows from the definition of SCSS, that $b$ can be eliminated from $\mathcal{D}(x_i)$ if and only if
for some $j$ such that $\{i,j\} \in E$, we have NotSnakeCovered($i,b,j$) $= \emptyset$.
The propagation algorithm, shown in Figure~\ref{fig:SCSSalgo}, uses a list ElimList
of triples $(i,b,j)$, where NotSnakeCovered($i,b,j$) has been found to be empty. 

\begin{figure} \centering
\fbox{\parbox{12cm}{
\begin{tabbing}
\hspace{4mm} \= \hspace{4mm} \= \hspace{4mm} \= \hspace{4mm} \= \hspace{4mm} \= \hspace{4mm} \= 
\hspace{4mm} \= \hspace{4mm} \= \hspace{4mm} \= \hspace{4mm} \= \hspace{-47mm} 
while (ElimList $\neq \emptyset$)  : \\
\> pop $(r,u,t)$ from NSlist ; \\
\> if $u \in \mathcal{D}(x_p) \ \land \ \text{NotSnakeCovered}(r,u,t) = \emptyset$ then \ \  *** SCSS still valid *** \\
\> \> delete $u$ fom $\mathcal{D}(x_r)$ ; \\
\> \>  *** Update NbBlocks and propagate *** \\
\> \> for all $k$ such that $\{k,r\} \in E$ : \\
\> \> \> for all $d,e \in \mathcal{D}(x_k)$ such that $(d,u) \in R_{kr}$ and $(e,u) \notin R_{kr}$ : .............(1) \\
\> \> \> \> NbBlocks($k,d,e,r$) := NbBlocks($k,d,e,r$) $- 1$ ; \\
\> \> \> \> if NbBlocks($k,d,e,r$) $= 0$ then \\
\> \> \> \> \> BlockVars($k,d,e$) := BlockVars($k,d,e$) $\setminus \{r\}$ ; \\
\> \> \> \> \> if BlockVars($k,d,e$) becomes $\emptyset$ then \ \ \ \  *** it was $\{r\}$ *** \\
\> \> \> \> \> \> for all $i \neq r$ such that $\{i,k\} \in E$ : \\
\> \> \> \> \> \> \> for all $a \in \mathcal{D}(x_i)$ such that $(a,e) \in R_{ik}$ : .....................(2) \\
\> \> \> \> \> \> \> \> IncNbSubs($i,a,k,d$) ; \ \ \ *** since $e$ is a new sub *** \\
\> \> \> \> \> if BlockVars($k,e,d$) becomes a singleton $\{i\}$ then  \\
\> \> \> \> \> \> for all $a \in \mathcal{D}(x_i)$ such that $(a,e) \in R_{ik}$ :  ............................(3) \\ 
\> \> \> \> \> \> \> IncNbSubs($i,a,k,d$) ; \ \ \ \ \ *** since $e$ is a new sub *** \\
\> \>  *** Update NbSubs and propagate *** \\
\> \> for all $i$ such that $\{i,r\} \in E$ : \\
\> \> \> for all $d \in \mathcal{D}(x_r) \setminus \{u\}$ : \\
\> \> \> \> for all $a \in \mathcal{D}(x_i)$ such that $(a,u) \in R_{ir}$ : ........................................(4) \\  
\> \> \> \> \> if BlockVars($r,u,d$) $\subseteq \{i\}$ then \\
\> \> \> \> \> \> DecNbSubs($i,a,r,d$) ; \ \  *** since sub $u$ has been eliminated ** \\  
\> \>  *** Update NbStops and propagate *** \\
\> \> for all $i$ such that $\{i,r\} \in E$ : \\
\> \> \> for all $a,b \in \mathcal{D}(x_i)$ such that $(b,u) \in R_{ir}$ and $(a,u) \notin R_{ir}$ : ..............(5)  \\
\> \> \> \>  if NbSubs($i,a,r,u$) $= 0$ then : \\
\> \> \> \> \>  DecNbStops($i,a,b,r$) ; \ \ \ *** since stop $u$ has been eliminated *** \\
\> \>  *** Update NbSnakeCovers and propagate *** \\
\> \> for all $j$ such that $\{j,r\} \in E$ : \\
\> \> \> for all $c \in \mathcal{D}(x_j)$ such that $(u,c) \in R_{rj} \lor \text{NbSubs}(r,u,j,c) > 0$ : \\
\> \> \> \> for all $b \in \mathcal{D}(x_r) \setminus \{u\}$ :  ...................................................................(6) \\
\> \> \> \> \> DecNbSnakeCovers($r,b,j,c$) ; \\
\> \>  *** Update  NotSnakeCovered *** \\
\> \> for all $i$ such that $\{i,r\} \in E$ :\\
\> \> \> for all $b \in \mathcal{D}(x_i)$ : \\
\> \> \> \> if $(b,u) \in R_{ir}$ and $\text{NbSnakeCovers}(i,b,r,u) = 0$ : ........................(7) \\
\> \> \> \> \> NotSnakeCovered($i,b,r$) := NotSnakeCovered($i,b,r$) $\setminus \{u\}$ ; \\
\> \> \> \> \> if NotSnakeCovered($i,b,r$) becomes $\emptyset$ then \ \ \ *** SCSS *** \\
\> \> \> \> \> \> add $(i,b,r)$ to ElimList ; 
\end{tabbing}
}}
\caption{The propagation algorithm for applying SCSS until convergence.}  \label{fig:SCSSalgo}
\end{figure}

\begin{figure} \centering
\fbox{\parbox{12cm}{
\begin{tabbing}
\hspace{4mm} \= \hspace{4mm} \= \hspace{4mm} \= \hspace{4mm} \= \hspace{4mm} \= \hspace{4mm} \= 
\hspace{4mm} \= \hspace{4mm} \= \hspace{4mm} \= \hspace{4mm} \= \hspace{-47mm} 
procedure IncNbSubs($i,a,k,d$) : \\
\> NbSubs($i,a,k,d$) := NbSubs($i,a,k,d$) $+ 1$ ; \\
\> if NbSubs($i,a,k,d$) $= 1$ then \ \ \ \ *** NbSubs($i,a,k,d$) becomes non-0 *** \\
\> \> for all $b \in \mathcal{D}(x_i)$ such that $(b,d) \in R_{ik}$ and $(a,d) \notin R_{ik}$ : \\
\> \> \> DecNbStops($i,a,b,k$) ;  \ \ \ \ *** since $d$ is no longer a stop *** \\
\> if $(a,d) \notin R_{ik}$ then \\
\> \> for all $b \in \mathcal{D}(x_i) \setminus \{a\}$ such that SopVars($i,a,b$) $\subseteq \{k\}$ : \\
\> \> \> IncNbSnakeCovers($i,b,k,d$) ; 
\end{tabbing}
\begin{tabbing}
\hspace{4mm} \= \hspace{4mm} \= \hspace{4mm} \= \hspace{4mm} \= \hspace{4mm} \= \hspace{4mm} \= 
\hspace{4mm} \= \hspace{4mm} \= \hspace{4mm} \= \hspace{4mm} \= \hspace{-47mm} 
procedure DecNbSubs($i,a,k,d$) : \\
\> NbSubs($i,a,k,d$) := NbSubs($i,a,k,d$) $- 1$ ; \\
\> if NbSubs($i,a,k,d$) $= 0$ then \\
\> \> for all $b \in \mathcal{D}(x_i)$ such that $(b,d) \in R_{ik}$ and $(a,d) \notin R_{ik}$  : \\
\> \> \> IncNbStops($i,a,b,k$) ; \\
\> if $(a,d) \notin R_{ik}$ then \\
\> \> for all $b \in \mathcal{D}(x_i) \setminus \{a\}$ such that StopVars($i,a,b$) $\subseteq \{k\}$ : \\
\> \> \> DecNbSnakeCovers($i,b,k,d$) ; 
\end{tabbing}
\begin{tabbing}
\hspace{4mm} \= \hspace{4mm} \= \hspace{4mm} \= \hspace{4mm} \= \hspace{4mm} \= \hspace{4mm} \= 
\hspace{4mm} \= \hspace{4mm} \= \hspace{4mm} \= \hspace{4mm} \= \hspace{-47mm} 
procedure IncNbStops($i,a,b,k$) : \\
\> NbStops($i,a,b,k$) := NbStops($i,a,b,k$) $+ 1$ ; \\
\> if NbStops($i,a,b,k$) $= 1$ then \ \ \ \ *** NbStops($i,a,b,k$) becomes non-0 *** \\
\> \> add $k$ to StopVars($i,a,b$) ; \\
\> \> if StopVars($i,a,b$) becomes $\{j,k\}$ for some $j \neq k$ then \ \ *** it is not $\{k\}$ *** \\
\> \> \> for all $c \in \mathcal{D}(x_j)$ such that $(a,c) \in R_{ij} \lor \text{NbSubs}(i,a,j,c) > 0$ : \\
\> \> \> \> DecNbSnakeCovers($i,b,j,c$) ; \\
\> \> if StopVars($i,a,b$) becomes $\{k\}$ then \ \ \ *** it is no longer empty *** \\
\> \> \> for all $j \neq k,i$ : \\
\> \> \> \> for all $c \in \mathcal{D}(x_j)$ such that $(a,c) \in R_{ij} \lor \text{NbSubs}(i,a,j,c) > 0$ : \\
\> \> \> \> \> DecNbSnakeCovers($i,b,j,c$) ; 
\end{tabbing}
\begin{tabbing}
\hspace{4mm} \= \hspace{4mm} \= \hspace{4mm} \= \hspace{4mm} \= \hspace{4mm} \= \hspace{4mm} \= 
\hspace{4mm} \= \hspace{4mm} \= \hspace{4mm} \= \hspace{4mm} \= \hspace{-47mm} 
procedure DecNbStops($i,a,b,k$) : \\
\> NbStops($i,a,b,k$) := NbStops($i,a,b,k$) $- 1$ ; \\
\> if NbStops($i,a,b,k$) $= 0$ then  \\
\> \> delete $k$ from StopVars($i,a,b$) ; \\
\> \> if StopVars($i,a,b$) becomes a singleton $\{j\}$ then  \\
\> \> \> for all $c \in \mathcal{D}(x_j)$ such that $(a,c) \in R_{ij} \lor \text{NbSubs}(i,a,j,c) > 0$ : \\
\> \> \> \> IncNbSnakeCovers($i,b,j,c$) ; \\
\> \> if StopVars($i,a,b$) becomes $\emptyset$ then \ \ \ *** StopVars($i,a,b$) was $\{k\}$ *** \\
\> \> \> for all $j \neq k,i$ : \\
\> \> \> \> for all $c \in \mathcal{D}(x_j)$ such that $(a,c) \in R_{ij} \lor \text{NbSubs}(i,a,j,c) > 0$ : \\
\> \> \> \> \> IncNbSnakeCovers($i,b,j,c$) ; 
\end{tabbing}
}}
\caption{Subprograms IncNbSubs, DecNbSubs, IncNbStops and DecNbStops used in the propagation of SCSS eliminations.}  \label{fig:SCSSsubprograms1}
\end{figure}

\begin{figure} \centering
\fbox{\parbox{12cm}{
\begin{tabbing}
\hspace{4mm} \= \hspace{4mm} \= \hspace{4mm} \= \hspace{4mm} \= \hspace{4mm} \= \hspace{4mm} \= 
\hspace{4mm} \= \hspace{4mm} \= \hspace{4mm} \= \hspace{4mm} \= \hspace{-47mm} 
procedure IncNbSnakeCovers($i,b,j,c$) : \\
\> NbSnakeCovers($i,b,j,c$) := NbSnakeCovers($i,b,j,c$) $+ 1$ ; \\
\> if NbSnakeCovers($i,b,j,c$) $= 1$ then \ \ \ \ *** it  becomes non-0 *** \\
\> \> if $(b,c) \in R_{ij}$ then \\
\> \> \> delete $c$ from NotSnakeCovered($i,b,j$) ; \\
\> \> \> if NotSnakeCovered($i,b,j$) becomes $\emptyset$ then   \\
\> \> \> \> add ($i,b,j$) to ElimList ; \ \ \ \ *** $b \in\mathcal{D}(x_i)$ satisfies SCSS ***
\end{tabbing}
\begin{tabbing}
\hspace{4mm} \= \hspace{4mm} \= \hspace{4mm} \= \hspace{4mm} \= \hspace{4mm} \= \hspace{4mm} \= 
\hspace{4mm} \= \hspace{4mm} \= \hspace{4mm} \= \hspace{4mm} \= \hspace{-47mm} 
procedure DecNbSnakeCovers($i,b,j,c$) : \\
\> NbSnakeCovers($i,b,j,c$) := NbSnakeCovers($i,b,j,c$) $- 1$ ; \\
\> if NbSnakeCovers($i,b,j,c$) $= 0$ and $(b,c) \in R_{ij}$ then \\
\> \> add $c$ to NotSnakeCovered($i,b,j$) ;
\end{tabbing}
}}
\caption{Subprograms IncNbSnakeCovers and
DecNbSnakeCovers used in the propagation of SCSS eliminations.}  \label{fig:SCSSsubprograms2}
\end{figure}

It is easy to see that the above data structures can be calculated in $O(ed^3)$ time and that they
require $O(ed^2)$ space. Hence, the existence of a possible value elimination by SCSS
can be checked in $O(ed^3)$ time and $O(ed^2)$ space

However, once a value $b$ has been eliminated
from some domain $\mathcal{D}(x_i)$ this can affect the validity of any other SCSS value-eliminations that have been
detected, meaning that the above data structures have to be updated for those arguments which might be concerned
by the elimination of $b$ from $\mathcal{D}(x_i)$. The propagation algorithm for updating
these data structures after eliminations by SCSS is given in Figure~\ref{fig:SCSSalgo}
with its subprograms given in Figures~\ref{fig:SCSSsubprograms1} and \ref{fig:SCSSsubprograms2}.

The subprograms IncNbSnakeCovers and DecNbSnakeCovers (Figure~\ref{fig:SCSSsubprograms2}) 
both have time complexity $O(1)$. Thus, 
the subprograms IncNbStops and DecNbStops both have time complexity $O(nd)$, which thus implies an
upper bound of $O(nd^2)$ for each call of IncNbSubs and DecNbSubs
(Figure~\ref{fig:SCSSsubprograms1}). For each $k,e,d$,
BlockVars($k,e,d$) is a monotonic decreasing set; it can therefore become empty (or a singleton) at most once
during propagation of SCSS eliminations (Figure~\ref{fig:SCSSalgo}).
This implies that the total number of calls to IncNbSubs is $O(ed^3)$ (in the loops at lines (2) and (3)
of the propagation algorithm of Figure~\ref{fig:SCSSalgo}).
Since a value $u$ can be eliminated from $\mathcal{D}(x_r)$ at most once, we also have an upper bound
of $O(ed^3)$ for the number of iterations of the loops in lines (1), (4), (5), (6) and (7), each
of which makes one call to one of the subprograms whose complexity we have just studied.
Putting all this together, we obtain an upper bound  of $O(end^5)$ for the time complexity of
applying SCSS until convergence.
We state formally in the following theorem what we have just proved.

\begin{theorem}
It is possible to verify in $O(ed^3)$ time and $O(ed^2)$ space whether or not
any value eliminations by SCSS can be performed on a binary CSP instance. Value eliminations by SCSS can then be applied 
until convergence in $O(end^5)$ time and $O(ed^2)$ space.
\end{theorem}

\section{Optimal sequences of eliminations}    \label{sec:nphard}

It is known that applying different sequences of neighbourhood operations until convergence 
produces isomorphic instances~\cite{ns}. 
This is not the case for CNS, SS or SCSS. Indeed, as we show in this section, the problems of maximising the
number of value-eliminations by CNS, SS or SCSS are all NP-hard.
These intractability results do not detract from the utility of these
operations, since any number of value eliminations reduces search-space size regardless
of whether or not this number is optimal.

\begin{theorem}
Finding the longest sequence of CNS value-eliminations or SCSS value-eliminations is NP-hard.
\end{theorem}

\begin{proof}
We prove this by givng a polynomial reduction from the set cover problem~\cite{DBLP:conf/coco/Karp72},
the well-known NP-complete problem which, given sets $S_1,\ldots,S_m  \subseteq U$
and an integer $k$, consists in determining whether there are $k$ sets $S_{i_1},\ldots,S_{i_k}$ which
cover $U$  (i.e. such that $S_{i_1} \cup \ldots \cup S_{i_k} = U$). We can assume that  $S_1 \cup \ldots \cup S_m = U$
and $k < m$, otherwise the problem is trivially solvable.
Given sets $S_1,\ldots,S_m  \subseteq U$, we create a 2-variable CSP instance with $\mathcal{D}(x_1) = \{1,\ldots,m\}$,
$\mathcal{D}(x_2) = U$ and $R_{12} = \{ (i,u) \mid u \in S_i \}$. We can eliminate value $i$ from $\mathcal{D}(x_1)$
by CNS (with, of course, $x_2$ as the conditioning variable) if and only if 
$S_1,\ldots,S_{i-1},S_{i+1},\ldots,m$ cover $U$. Indeed, we can continue eliminating elements from
$\mathcal{D}(x_1)$ by CNS provided the sets $S_j$ ($j \in \mathcal{D}(x_1)$) still cover $U$.
Clearly, maximising the number of eliminations from $\mathcal{D}(x_1)$ by CNS is equivalent to
minimising the size of the cover. To prevent any eliminations from the domain of $x_2$ by CNS, we 
add variables $x_3$ and $x_4$ with domains $\{1,\ldots,m\}$, together with the three equality constraints $x_2=x_3$,
$x_3=x_4$ and $x_4=x_2$.
To complete the proof for CNS, it is sufficient to observe that this reduction is polynomial.

It is easily verified that in this instance, CNS and SCSS are equivalent. Hence, this proof also
shows that finding the longest sequence of SCSS value-eliminations is NP-hard.
\end{proof}

In the proof of the following theorem, we need the following notion: 
we say that a sequence of value-eliminations by snake-substitution (SS) is \emph{convergent}
if no more SS value-eliminations are possible after this sequence of eliminations is applied.

\begin{theorem}
Finding a longest sequence of snake-substitution value-eliminations 
is NP-hard.
\end{theorem}

\begin{proof}
It suffices to demonstrate a polynomial reduction from the problem {\sc Max 2-Sat} which is known to
be NP-hard~\cite{DBLP:journals/tcs/GareyJS76}. 
Consider an instance $I_{2SAT}$ of {\sc Max 2-Sat} with variables $X_1,\ldots,X_N$ and $M$ binary clauses:
the goal is to find a truth assignment to these variables which maximises the number of satisfied clauses.
We will construct a binary CSP instance $I_{CSP}$ on $O(N+M)$ variables, each with domain of size
at most four, such that the convergent sequences $S$ of SS value-eliminations in $I_{CSP}$ correspond
to truth assignments to $X_1,\ldots,X_N$ and the length of $S$ is $\alpha N + \beta m$
where $\alpha, \beta$ are constants and $m$ is the number of clauses of $I_{2SAT}$ satisfied by the
corresponding truth assignment.

We require four constructions (which we explain in detail below):
\begin{enumerate}
\item the construction in Figure~\ref{fig:snake-propagation} simulates a {\sc Max 2-Sat} literal $X$ by a
path of CSP variables joined by greater-than-or-equal-to constraints.
\item the construction in Figure~\ref{fig:snake-negation} simulates the relationship between a {\sc Max 2-Sat}
variable $X$ and its negation $\overline{X}$.
\item the construction in Figure~\ref{fig:snake-copies} allows us to create multiple copies of a {\sc Max 2-Sat} literal $X$.
\item the construction in Figure~\ref{fig:snake-clause} simulates a binary clause $X \vee Y$ where
$X,Y$ are {\sc Max 2-Sat} literals.
\end{enumerate}

In each of these figures, each oval represents a CSP variable with the bullets inside the oval representing the 
possible values for this variable. If there is a non-trivial constraint between two variables $x_i,x_j$ this is represented
by joining up with a line those pairs of values $a,b$ such that $(a,b) \in R_{ij}$. Where the constraint has a 
compact form, such as $x_1 \geq x_2$ this is written next to the constraint. 
In the following, we write $b \overset{x_i}{\rightsquigarrow} a$ if 
$b \in \mathcal{D}(x_i)$ is snake substitutable by $a \in \mathcal{D}(x_i)$.
Our constructions are such that the only value that can be eliminated from any domain by SS is the value $2$.

\thicklines \setlength{\unitlength}{1.6pt}
\newsavebox{\varthreeb}
\savebox{\varthreeb}(20,40){
\begin{picture}(20,40)(0,0)
\put(10,20){\oval(18,36)} 
\put(10,10){\makebox(0,0){{$\bullet$}}}
\put(10,20){\makebox(0,0){{$\bullet$}}} 
\put(10,30){\makebox(0,0){{$\bullet$}}}
\end{picture}
}
\thicklines \setlength{\unitlength}{1.3pt}
\newsavebox{\varthree}
\savebox{\varthree}(20,40){
\begin{picture}(20,40)(0,0)
\put(10,20){\oval(18,38)} 
\put(10,10){\makebox(0,0){{$\bullet$}}}
\put(10,20){\makebox(0,0){{$\bullet$}}} 
\put(10,30){\makebox(0,0){{$\bullet$}}}
\end{picture}
}
\newsavebox{\vartwo}
\savebox{\vartwo}(20,40){
\begin{picture}(20,40)(0,0)
\put(10,20){\oval(18,38)} 
\put(10,10){\makebox(0,0){{$\bullet$}}}
\put(10,30){\makebox(0,0){{$\bullet$}}}
\end{picture}
}
\newsavebox{\varfour}
\savebox{\varfour}(20,50){
\begin{picture}(20,50)(0,0)
\put(10,25){\oval(18,48)} 
\put(10,10){\makebox(0,0){{$\bullet$}}}
\put(10,20){\makebox(0,0){{$\bullet$}}} 
\put(10,30){\makebox(0,0){{$\bullet$}}}
\put(10,40){\makebox(0,0){{$\bullet$}}}
\end{picture}
}

\thicklines
\setlength{\unitlength}{1.3pt}
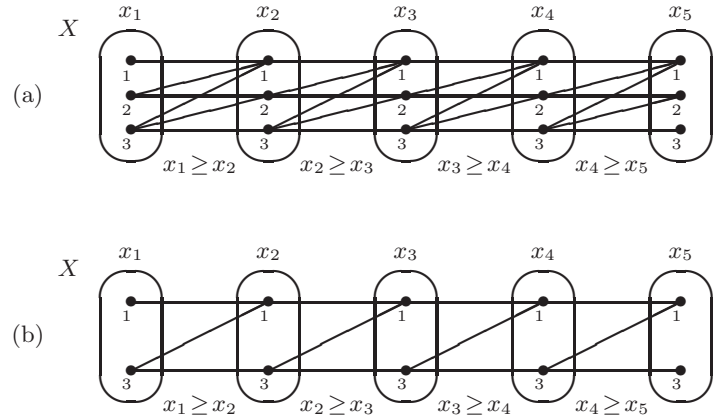
\begin{figure}
\centering
\begin{picture}(193,125)(0,-5)
\put(0,70){
\begin{picture}(193,50)(-13,0)
\put(-20,20){\makebox(0,0){{(a)}}}
\put(0,0){\usebox{\varthree}} \put(40,0){\usebox{\varthree}} 
\put(80,0){\usebox{\varthree}} \put(120,0){\usebox{\varthree}}  \put(160,0){\usebox{\varthree}} 
\put(10,10){\line(1,0){160}} \put(10,20){\line(1,0){160}} \put(10,30){\line(1,0){160}}
\put(10,10){\line(4,1){80}} \put(50,10){\line(4,1){80}} \put(90,10){\line(4,1){80}}
\put(10,10){\line(2,1){40}} \put(50,10){\line(2,1){40}} \put(90,10){\line(2,1){40}} \put(130,10){\line(2,1){40}}
\put(10,20){\line(4,1){40}} \put(130,10){\line(4,1){40}}
\put(9,26){\makebox(0,0){{$_1$}}} \put(9,16){\makebox(0,0){{$_2$}}} \put(9,6){\makebox(0,0){{$_3$}}}
\put(49,26){\makebox(0,0){{$_1$}}} \put(49,16){\makebox(0,0){{$_2$}}} \put(49,6){\makebox(0,0){{$_3$}}}
\put(89,26){\makebox(0,0){{$_1$}}} \put(89,16){\makebox(0,0){{$_2$}}} \put(89,6){\makebox(0,0){{$_3$}}}
\put(129,26){\makebox(0,0){{$_1$}}} \put(129,16){\makebox(0,0){{$_2$}}} \put(129,6){\makebox(0,0){{$_3$}}}
\put(169,26){\makebox(0,0){{$_1$}}} \put(169,16){\makebox(0,0){{$_2$}}} \put(169,6){\makebox(0,0){{$_3$}}}
\put(10,44){\makebox(0,0){{$x_1$}}}  \put(50,44){\makebox(0,0){{$x_2$}}}
\put(90,44){\makebox(0,0){{$x_3$}}}  \put(130,44){\makebox(0,0){{$x_4$}}} \put(170,44){\makebox(0,0){{$x_5$}}}
\put(30,0){\makebox(0,0){{$x_1 \! \geq \! x_2$}}} \put(70,0){\makebox(0,0){{$x_2 \! \geq \! x_3$}}}
\put(110,0){\makebox(0,0){{$x_3 \! \geq \! x_4$}}} \put(150,0){\makebox(0,0){{$x_4 \! \geq \! x_5$}}}
\put(-8,40){\makebox(0,0){{$X$}}}
\end{picture}}
\put(0,0){
\begin{picture}(193,50)(-13,0)
\put(-20,20){\makebox(0,0){{(b)}}}
\put(0,0){\usebox{\vartwo}} \put(40,0){\usebox{\vartwo}} 
\put(80,0){\usebox{\vartwo}} \put(120,0){\usebox{\vartwo}}  \put(160,0){\usebox{\vartwo}} 
\put(10,10){\line(1,0){160}}  \put(10,30){\line(1,0){160}}
\put(10,10){\line(2,1){40}} \put(50,10){\line(2,1){40}} \put(90,10){\line(2,1){40}} \put(130,10){\line(2,1){40}}
\put(9,26){\makebox(0,0){{$_1$}}}  \put(9,6){\makebox(0,0){{$_3$}}}
\put(49,26){\makebox(0,0){{$_1$}}}  \put(49,6){\makebox(0,0){{$_3$}}}
\put(89,26){\makebox(0,0){{$_1$}}}  \put(89,6){\makebox(0,0){{$_3$}}}
\put(129,26){\makebox(0,0){{$_1$}}}  \put(129,6){\makebox(0,0){{$_3$}}}
\put(169,26){\makebox(0,0){{$_1$}}}  \put(169,6){\makebox(0,0){{$_3$}}}
\put(10,44){\makebox(0,0){{$x_1$}}}  \put(50,44){\makebox(0,0){{$x_2$}}}
\put(90,44){\makebox(0,0){{$x_3$}}}  \put(130,44){\makebox(0,0){{$x_4$}}} \put(170,44){\makebox(0,0){{$x_5$}}}
\put(30,0){\makebox(0,0){{$x_1 \! \geq \! x_2$}}} \put(70,0){\makebox(0,0){{$x_2 \! \geq \! x_3$}}}
\put(110,0){\makebox(0,0){{$x_3 \! \geq \! x_4$}}} \put(150,0){\makebox(0,0){{$x_4 \! \geq \! x_5$}}}
\put(-8,40){\makebox(0,0){{$X$}}}
\end{picture}}
\end{picture}
\caption{A construction to simulate a {\sc Max 2-Sat} variable $X$: (a) $X=0$, (b) $X=1$.} \label{fig:snake-propagation}
\end{figure}

Figure~\ref{fig:snake-propagation}(a) shows a path of CSP variables constrained by greater-than-or-equal-to
constraints. The end variables $x_1$ and $x_5$ are constrained by other variables that, for clarity
of presentation, are not shown in this figure. If value $2$ is eliminated from $\mathcal{D}(x_1)$, then
we have $2 \overset{x_2}{\rightsquigarrow} 3$. In fact, $2$ is neighbourhood substitutable by $3$.
Once the value $2$ is eliminated from $\mathcal{D}(x_2)$,
we have $2 \overset{x_3}{\rightsquigarrow} 3$. Indeed, eliminations of the value $2$ propagate so that
in the end we have the situation shown in Figure~\ref{fig:snake-propagation}(b). 
By a symmetrical argument, the elimination of the value $2$ from $\mathcal{D}(x_5)$
propagates from right to left (this time by neighbourhood substitution by $1$) to again produce the situation 
shown in Figure~\ref{fig:snake-propagation}(b). It is easily verified that,
without any eliminations from the domains $\mathcal{D}(x_1)$ or $\mathcal{D}(x_5)$, no values for the
variables $x_2,x_3,x_4$ are snake-substitutable. Furthermore, the values $1$ and $3$ for the 
variables $x_2,x_3,x_4$ are not snake-substitutable even after the elimination of the value $2$
from all domains. So we either have no eliminations, which we associate with the truth assignment $X=0$
(where $X$ is the {\sc Max 2-Sat} literal corresponding to this path of variables in $I_{CSP}$)
or the value $2$ is eliminated from all domains, which we associate with the truth assignment $X=1$.

\thicklines
\setlength{\unitlength}{1.3pt}
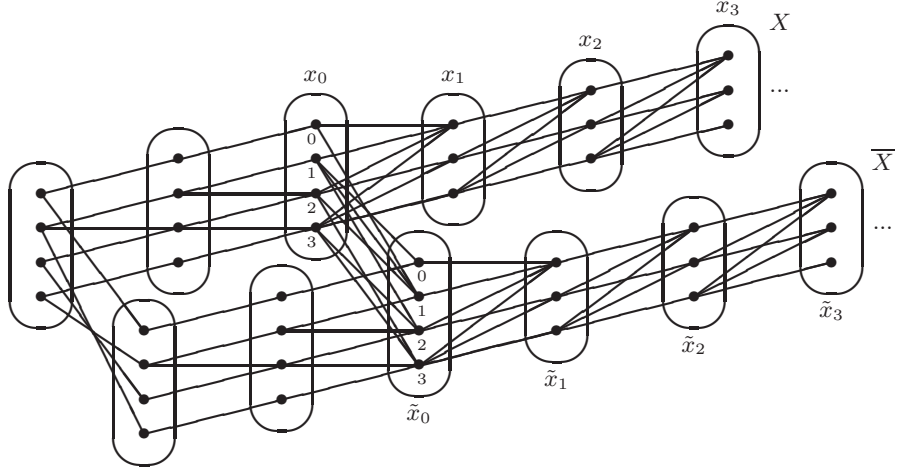
\begin{figure}
\centering
\begin{picture}(260,150)(0,0)
\put(0,50){\usebox{\varfour}} \put(40,60){\usebox{\varfour}} \put(80,70){\usebox{\varfour}}
\put(120,80){\usebox{\varthree}} \put(160,90){\usebox{\varthree}} \put(200,100){\usebox{\varthree}}
\put(30,10){\usebox{\varfour}} \put(70,20){\usebox{\varfour}} \put(110,30){\usebox{\varfour}}
\put(150,40){\usebox{\varthree}} \put(190,50){\usebox{\varthree}} \put(230,60){\usebox{\varthree}}
\put(10,60){\line(4,1){200}} \put(10,70){\line(4,1){200}} \put(10,80){\line(4,1){200}} \put(10,90){\line(4,1){80}}
\put(40,20){\line(4,1){200}} \put(40,30){\line(4,1){200}} \put(40,40){\line(4,1){200}} \put(40,50){\line(4,1){80}}
\put(10,80){\line(1,0){80}} \put(50,90){\line(1,0){40}} 
\put(40,40){\line(1,0){80}} \put(80,50){\line(1,0){40}}
\put(90,80){\line(4,1){40}} \put(90,80){\line(2,1){40}} \put(90,80){\line(4,3){40}} \put(90,90){\line(2,1){40}} \put(90,110){\line(1,0){40}}
\put(120,40){\line(4,1){40}} \put(120,40){\line(2,1){40}} \put(120,40){\line(4,3){40}} \put(120,50){\line(2,1){40}} \put(120,70){\line(1,0){40}}
\put(10,60){\line(3,-2){30}} \put(10,70){\line(3,-4){30}} \put(10,80){\line(1,-2){30}} \put(10,90){\line(3,-4){30}}
\put(90,80){\line(3,-4){30}} \put(90,80){\line(1,-1){30}} \put(90,90){\line(3,-5){30}} \put(90,90){\line(1,-1){30}}
\put(90,100){\line(3,-5){30}} \put(90,100){\line(3,-4){30}} \put(90,100){\line(1,-1){30}} \put(90,110){\line(3,-5){30}}
\put(130,90){\line(2,1){80}} \put(130,90){\line(4,3){40}} \put(130,100){\line(2,1){40}} 
\put(170,100){\line(2,1){40}} \put(170,100){\line(4,3){40}}
\put(160,50){\line(2,1){80}} \put(160,50){\line(4,3){40}} \put(160,60){\line(2,1){40}} 
\put(200,60){\line(2,1){40}} \put(200,60){\line(4,3){40}}
\put(225,120){\makebox(0,0){{$...$}}} \put(255,80){\makebox(0,0){{$...$}}}
\put(225,140){\makebox(0,0){{$X$}}} \put(255,100){\makebox(0,0){{$\overline{X}$}}}
\put(90,124){\makebox(0,0){{$x_0$}}} \put(130,124){\makebox(0,0){{$x_1$}}}
\put(170,134){\makebox(0,0){{$x_2$}}} \put(210,144){\makebox(0,0){{$x_3$}}}
\put(120,26){\makebox(0,0){{$\tilde{x}_0$}}} \put(160,36){\makebox(0,0){{$\tilde{x}_1$}}}
\put(200,46){\makebox(0,0){{$\tilde{x}_2$}}} \put(240,56){\makebox(0,0){{$\tilde{x}_3$}}}
\put(89,106){\makebox(0,0){{$_0$}}} \put(89,96){\makebox(0,0){{$_1$}}} 
\put(89,86){\makebox(0,0){{$_2$}}} \put(89,76){\makebox(0,0){{$_3$}}}
\put(121,66){\makebox(0,0){{$_0$}}} \put(121,56){\makebox(0,0){{$_1$}}} 
\put(121,46){\makebox(0,0){{$_2$}}} \put(121,36){\makebox(0,0){{$_3$}}}
\end{picture}
\caption{A construction to simulate a {\sc Max 2-Sat} variable $X$ and its negation $\overline{X}$.} \label{fig:snake-negation}
\end{figure}

The construction in Figure~\ref{fig:snake-negation} joins the two path-of-CSP-variables constructions corresponding to
the literals $X$ and $\overline{X}$. This construction ensures that exactly one of $X$ and $\overline{X}$ are assigned the value $1$.
It is easy (if tedious) to verify that the only snake substitutions that are possible in this construction
are $2 \overset{x_0}{\rightsquigarrow} 3$ and $2 \overset{\tilde{x}_0}{\rightsquigarrow} 3$, but that
after elimination of the value $2$ from either of $\mathcal{D}(x_0)$ or $\mathcal{D}(\tilde{x}_0)$,
the other snake substitution is no longer valid. Once, for example, $2$ has been eliminated from 
$\mathcal{D}(x_0)$, then this elimination propagates along the path of CSP variables ($x_1, x_2, x_3, \ldots$) corresponding
to $X$, as shown in Figure~\ref{fig:snake-propagation}(b). By a symmetrical argument, if $2$ is eliminated from 
$\mathcal{D}(\tilde{x}_0)$, then this elimination propagates along the path of CSP variables 
($\tilde{x}_1,\tilde{ x}_2,\tilde{x}_3, \ldots$) corresponding
to $\overline{X}$. Thus, this construction simulates the
assignment of a truth value to $X$ and its complement to $\overline{X}$.

\thicklines
\setlength{\unitlength}{1.6pt}
\begin{figure}
\centering
\begin{picture}(190,140)(0,-10)
\put(0,10){\usebox{\varthreeb}} \put(40,10){\usebox{\varthreeb}} 
\put(120,10){\usebox{\varthreeb}} \put(160,10){\usebox{\varthreeb}} 
\put(80,30){\usebox{\varthreeb}}  \put(40,50){\usebox{\varthreeb}} 
\put(120,50){\usebox{\varthreeb}} \put(160,50){\usebox{\varthreeb}}
\put(10,80){\usebox{\varthreeb}} \put(70,80){\usebox{\varthreeb}}
\put(10,20){\line(1,0){40}} \put(10,30){\line(1,0){40}} \put(10,40){\line(1,0){40}} 
\put(10,20){\line(4,1){40}} \put(10,20){\line(2,1){40}} \put(10,30){\line(4,1){40}} 
\put(130,20){\line(1,0){40}} \put(130,30){\line(1,0){40}} \put(130,40){\line(1,0){40}} 
\put(130,20){\line(4,1){40}} \put(130,20){\line(2,1){40}} \put(130,30){\line(4,1){40}} 
\put(130,60){\line(1,0){40}} \put(130,70){\line(1,0){40}} \put(130,80){\line(1,0){40}} 
\put(130,60){\line(4,1){40}} \put(130,60){\line(2,1){40}} \put(130,70){\line(4,1){40}} 
\put(50,20){\line(2,1){40}} \put(50,30){\line(2,1){40}} \put(50,40){\line(2,1){40}} 
\put(50,20){\line(4,3){40}} \put(50,20){\line(1,1){40}} \put(50,30){\line(4,3){40}}  
\put(90,40){\line(2,1){40}} \put(90,50){\line(2,1){40}} \put(90,60){\line(2,1){40}} 
\put(90,40){\line(4,3){40}} \put(90,40){\line(1,1){40}} \put(90,50){\line(4,3){40}} 
\put(50,60){\line(2,-1){40}} \put(50,70){\line(2,-1){40}} \put(50,80){\line(2,-1){40}} 
\put(50,70){\line(4,-3){40}} \put(50,80){\line(4,-3){40}} \put(50,80){\line(1,-1){40}} 
\put(90,40){\line(2,-1){40}} \put(90,50){\line(2,-1){40}} \put(90,60){\line(2,-1){40}} 
\put(90,40){\line(4,-1){40}} \put(90,40){\line(1,0){40}} \put(90,50){\line(4,-1){40}} 
\put(20,90){\line(1,0){60}} \put(20,100){\line(1,0){60}} \put(20,110){\line(1,0){60}} 
\put(20,90){\line(1,-1){30}} \put(20,100){\line(1,-1){30}} \put(20,110){\line(1,-1){30}}  
\put(50,60){\line(1,1){30}} \put(50,70){\line(1,1){30}} \put(50,80){\line(1,1){30}} 
\put(10,36){\makebox(0,0){{$_1$}}} \put(10,26){\makebox(0,0){{$_2$}}} \put(10,16){\makebox(0,0){{$_3$}}}
\put(50,36){\makebox(0,0){{$_1$}}} \put(50,26){\makebox(0,0){{$_2$}}} \put(50,16){\makebox(0,0){{$_3$}}}
\put(90,56){\makebox(0,0){{$_1$}}} \put(90,46){\makebox(0,0){{$_2$}}} \put(90,36){\makebox(0,0){{$_3$}}}
\put(130,36){\makebox(0,0){{$_1$}}} \put(130,26){\makebox(0,0){{$_2$}}} \put(130,16){\makebox(0,0){{$_3$}}}
\put(170,36){\makebox(0,0){{$_1$}}} \put(170,26){\makebox(0,0){{$_2$}}} \put(170,16){\makebox(0,0){{$_3$}}} 
\put(130,76){\makebox(0,0){{$_1$}}} \put(130,66){\makebox(0,0){{$_2$}}} \put(130,56){\makebox(0,0){{$_3$}}}
\put(170,76){\makebox(0,0){{$_1$}}} \put(170,66){\makebox(0,0){{$_2$}}} \put(170,56){\makebox(0,0){{$_3$}}}
\put(-5,30){\makebox(0,0){{$...$}}} \put(-5,45){\makebox(0,0){{$X$}}}
\put(185,70){\makebox(0,0){{$...$}}} \put(185,85){\makebox(0,0){{$X'$}}}
\put(186,30){\makebox(0,0){{$...$}}} \put(185,45){\makebox(0,0){{$X''$}}}
\put(10,6){\makebox(0,0){{$x_1$}}} \put(50,6){\makebox(0,0){{$x_2$}}} 
\put(90,26){\makebox(0,0){{$x_3$}}} 
\put(130,6){\makebox(0,0){{$x_4''$}}} \put(170,6){\makebox(0,0){{$x_5''$}}} 
\put(130,94){\makebox(0,0){{$x_4'$}}} \put(170,94){\makebox(0,0){{$x_5'$}}} 
\put(20,124){\makebox(0,0){{$u$}}} \put(80,124){\makebox(0,0){{$v$}}} \put(36,60){\makebox(0,0){{$w$}}} 
\put(30,12){\makebox(0,0){{$_{x_1 \geq x_2}$}}}
\put(150,12){\makebox(0,0){{$_{x_4'' \geq x_5''}$}}} \put(150,86){\makebox(0,0){{$_{x_4' \geq x_5'}$}}}
\put(73,23){\makebox(0,0){{$_{x_2 \geq x_3}$}}} \put(106,23){\makebox(0,0){{$_{x_3 \geq x_4''}$}}}
\put(79,72){\makebox(0,0){{$_{w \leq x_3}$}}} \put(106,76){\makebox(0,0){{$_{x_3 \geq x_4'}$}}}
\put(27,73){\makebox(0,0){{$_{u=w}$}}} \put(50,114){\makebox(0,0){{$_{u=v}$}}} \put(78,78){\makebox(0,0){{$_{w=v}$}}}

\end{picture}
\caption{A construction to create two copies $X'$ and $X''$ of the {\sc Max 2-Sat}
variable $X$.} \label{fig:snake-copies}
\end{figure}
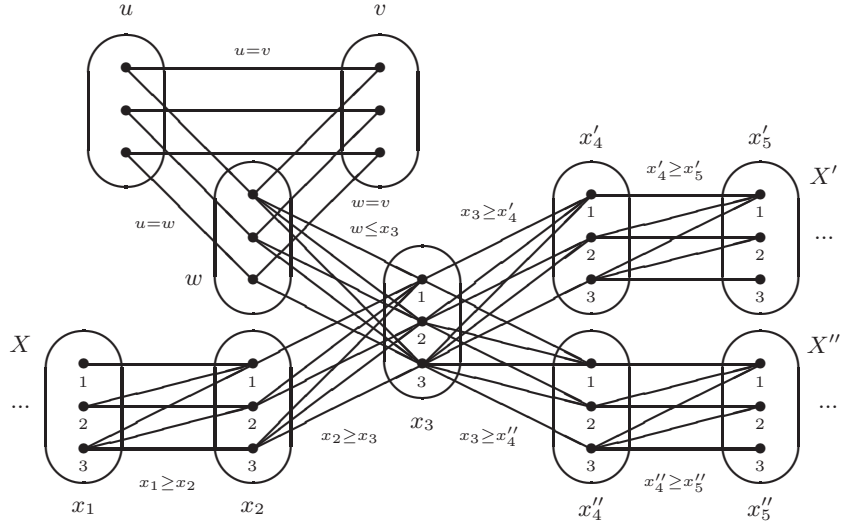

Since any literal of $I_{2SAT}$ may occur in several clauses, we need to be able to make copies of
any literal. Figure~\ref{fig:snake-copies} shows a construction that creates two copies $X',X''$
of a literal $X$. This construction can easily be generalised to make $k$ copies of a literal, if required,
by having $k$ identical paths of greater-than-equal-to constraints on the right of the figure
all starting at the pivot variable $x_3$.
Before any eliminations are performed, no snake substitutions are possible in this construction.
However, once the value $2$ has been eliminated from $\mathcal{D}(x_1)$, eliminations
propagate, as in Figure~\ref{fig:snake-propagation}: the value $2$ can successively be eliminated
from the domains of variables $x_2$, $x_3$, $x_4'$, $x_5'$, and $x_4''$, $x_5''$. 
Each elimination is in fact by neighbourhood substitution, as in Figure~\ref{fig:snake-propagation}. 
These eliminations mean that we effectively have two copies $X'$, $X''$ of the literal $X$.
The triangle of equality constraints at the top left of this construction is is there simply to prevent
propagation in the reverse direction: even if the value $2$ is eliminated from the domains of $x_5',x_4'$
and $x_5'',x_4''$ by the propagation of eliminations from the right, this cannot provoke the elimination of the
value $2$ from the domain of the pivot variable $x_3$.

\thicklines
\setlength{\unitlength}{1.3pt}
\begin{figure}
\centering
\begin{picture}(130,150)(0,0)
\put(10,10){\usebox{\varthree}} \put(50,10){\usebox{\varthree}} 
\put(10,90){\usebox{\varthree}} \put(50,90){\usebox{\varthree}} 
\put(100,70){\oval(38,18)} 
\put(90,70){\makebox(0,0){{$\bullet$}}}
\put(100,70){\makebox(0,0){{$\bullet$}}} 
\put(110,70){\makebox(0,0){{$\bullet$}}}
\put(20,20){\line(1,0){40}} \put(20,20){\line(4,1){40}} \put(20,20){\line(2,1){40}} 
\put(20,30){\line(1,0){40}} \put(20,30){\line(4,1){40}} \put(20,40){\line(1,0){40}} 
\put(19,36){\makebox(0,0){{$_1$}}} \put(19,26){\makebox(0,0){{$_2$}}} \put(19,16){\makebox(0,0){{$_3$}}}
\put(59,36){\makebox(0,0){{$_1$}}} \put(59,26){\makebox(0,0){{$_2$}}} \put(59,16){\makebox(0,0){{$_3$}}}
\put(20,100){\line(1,0){40}} \put(20,100){\line(4,1){40}} \put(20,100){\line(2,1){40}} 
\put(20,110){\line(1,0){40}} \put(20,110){\line(4,1){40}} \put(20,120){\line(1,0){40}} 
\put(19,116){\makebox(0,0){{$_1$}}} \put(19,106){\makebox(0,0){{$_2$}}} \put(19,96){\makebox(0,0){{$_3$}}}
\put(59,116){\makebox(0,0){{$_1$}}} \put(59,106){\makebox(0,0){{$_2$}}} \put(59,96){\makebox(0,0){{$_3$}}}
\put(5,30){\makebox(0,0){{$...$}}} \put(5,45){\makebox(0,0){{$Y$}}}
\put(5,110){\makebox(0,0){{$...$}}} \put(5,125){\makebox(0,0){{$X$}}}
\put(60,20){\line(1,1){50}} \put(60,20){\line(4,5){40}} \put(60,20){\line(3,5){30}} 
\put(60,30){\line(1,1){40}} \put(60,30){\line(3,4){30}} \put(60,40){\line(1,1){30}} 
\put(60,100){\line(5,-3){50}} \put(60,100){\line(4,-3){40}} \put(60,100){\line(1,-1){30}} 
\put(60,110){\line(5,-4){50}} \put(60,110){\line(1,-1){40}} \put(60,120){\line(1,-1){50}} 
\put(91,66){\makebox(0,0){{$_1$}}} \put(101,66){\makebox(0,0){{$_2$}}} \put(111,66){\makebox(0,0){{$_3$}}}
\put(20,134){\makebox(0,0){{$x_1$}}} \put(60,134){\makebox(0,0){{$x_2$}}}
\put(20,54){\makebox(0,0){{$y_1$}}} \put(60,54){\makebox(0,0){{$y_2$}}}
\put(124,70){\makebox(0,0){{$z$}}}
\put(40,90){\makebox(0,0){{$x_1 \! \geq \! x_2$}}} \put(40,10){\makebox(0,0){{$y_1 \! \geq \! y_2$}}}
\put(102,100){\makebox(0,0){{$x_2 \! \geq \! 4 \! - \! z$}}} \put(95,40){\makebox(0,0){{$y_2 \! \geq \! z$}}}

\end{picture}
\caption{A construction to simulate a {\sc Max 2-Sat} clause $X \vee Y$.} \label{fig:snake-clause}
\end{figure}
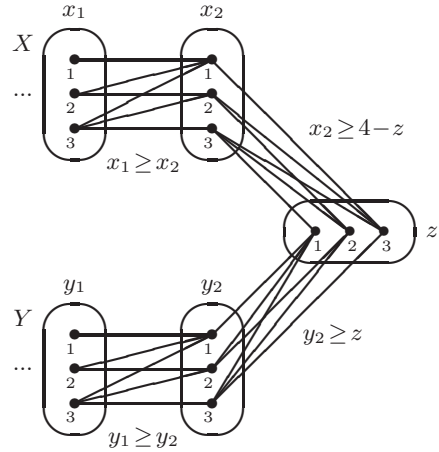

Finally, the construction of Figure~\ref{fig:snake-clause} simulates the clause $X \vee Y$. In fact, 
this construction simply joins together the paths of CSP-variables corresponding to the two literals $X,Y$, 
via a variable $z$. It is easily verified that the elimination of the value $2$ from the domain of $x_1$
allows the propagation of eliminations of the value $2$ from the domains of $x_2$, $z$, $y_2$, $y_1$
in exactly the same way as the propagation of eliminations in Figure~\ref{fig:snake-propagation}.
Similarly, the elimination of the value $2$ from the domain of $y_1$ propagates to all other variables
in the opposite order $y_2$, $z$, $x_2$, $x_1$. Thus, if one or other of the literals $X$ or $Y$ in the clause
is assigned $1$, then the value $2$ is eliminated from all domains of this construction. Eliminations can 
propagate back up to the pivot variable ($x_3$ in Figure~\ref{fig:snake-copies}) but no further, as explained
in the previous paragraph.

Putting all this together, we can see that there is a one-to-one correspondence between convergent sequences 
of SS value-eliminations and  
truth assignments to the variables of the {\sc Max 2-Sat} instance. Furthermore, the number of SS value-eliminations
is maximised when this truth assignment maximises the number of satisfied clauses, since it is $\alpha N + \beta m$
where $\alpha$ is the number of CSP-variables in each path of greater-than-or-equal-to constraints corresponding
to a literal, $\beta$ is the number of CSP-variables in each clause construction and $m$ is the number of satisfied
clauses. This reduction is clearly polynomial.
\end{proof}



\section{Discussion and Conclusion}

We have given two different value-elimination rules, namely snake substitutability (SS) and 
conditioned neighbourhood substitutability (CNS), which strictly subsume neighbourhood
substitution but nevertheless can be applied in the same $O(ed^3)$ time complexity. 
We have also given a more general
notion of substitution (SCSS) subsuming both these rules that can be detected in $O(ed^3)$ time.
The examples in Figure~\ref{fig:examples}  show that these three rules are strictly stronger than
neighborhood substitution and that SS and CNS are incomparable. 
We found snake substitution to be particularly effective when applied
to the problem of labelling line-drawings of polyhedral scenes.

Further research is required to investigate generalisations of SS, CNS or SCSS to non-binary or even global constraints.
Another obvious avenue of research is the generalisation to valued CSPs (also known as cost-function networks).
It is known that the generalisation of neighbourhood substitution to binary 
valued CSPs~\cite{Goldstein,DBLP:conf/cp/LecoutreRD12} can be applied to convergence 
in $O(ed^4)$ time if the aggregation operator is strictly monotonic or idempotent~\cite{DBLP:journals/fss/Cooper03}.
The notion of snake substitutability has already been generalised to binary valued CSPs
and it has been shown that it is possible to test this notion
in $O(ed^4)$ time if the aggregation operator is addition over the 
non-negative rationals (which is a particular example of a strictly monotonic operator)~\cite{DBLP:conf/cp/CooperJC18}. 
However, further research is required to determine the complexity of applying this operation until convergence.

It is known that it is possible to efficiently find all (or a given number of) 
solutions to a CSP after applying neighbourhood substitution:
given the set of all solutions to the reduced instance, it is possible to reconstruct
$K \geq 1$ solutions to the original instance $I$ (or to determine that $I$ does not have $K$ solutions) in
$O(K(de+n^2))$ time~\cite{ns}. This also holds for the case of conditioned neighbourhood substitution, since,
as for neighbourhood substitution, for each solution $s$ found and for each value $b$ eliminated from
some domain $\mathcal{D}(x_i)$, it suffices to test each putative solution obtained
by replacing $s_i$ by $b$. Unfortunately, the extra strength of snake substitution (SS) is here
a drawback, since, by exactly the same argument as for the $\exists2snake$ value-elimination rule
(which is a weaker version of SS)~\cite{jcss}, we can deduce that determining whether a binary CSP
instance has two or more solutions is NP-hard, even given the set of solutions to the reduced instance
after applying SS.

This work begs the interesting theoretical question as to the existence of reduction operations which strengthen 
other known reduction operations without increasing complexity.

\bibliographystyle{splncs04}
\bibliography{biblioStrengthening}

\end{document}